\documentclass[twoside]{article}

\usepackage[accepted]{aistats2016}
\usepackage{hyperref}
\usepackage{url}
\usepackage{amsmath, amsthm, bbm, subfig, algorithm, algorithmic, amsfonts, caption, enumerate, stmaryrd, ctable, cancel, soul, footnote, xspace, placeins, wrapfig}

\newcommand{\x}{{\mathbf x}}
\newcommand{\bE}{{\mathbb E}}

\newcommand{\loco}{{\sc Loco}\xspace}
\newcommand{\locod}{{\sc Dual-Loco}\xspace}
\newcommand{\cocoa}{{\sc CoCoA}\xspace}
\newcommand{\cocoap}{{\sc CoCoA}$^+$\xspace}

\newcommand{\rank}{r}

\newcommand{\samp}{n}
\newcommand{\dims}{p}
\newcommand{\dimset}{\mathcal{P}}
\newcommand{\minusk}{(-k)}

\newcommand{\blocks}{K}
\newcommand{\dimsk}{\tau}

\newcommand{\dimsks}{\tau_{subs}}
\newcommand{\lenlambda}{l}

\newcommand{\RP}{\boldsymbol{\Pi}}
\newcommand{\feats}{\bar{\Xt}}

\newcommand{\alphatil}{\tilde{\boldalpha}}
\newcommand{\alphaopt}{{\boldalpha}^{*}}
\newcommand{\gammatil}{\tilde{\gamma}}
\newcommand{\gammaopt}{{\gamma}^{*}}

\newcommand{\soln}{\boldbeta}
\newcommand{\solnRP}{\bar{\soln}}

\newcommand{\solnopt}{\soln^{*}}

\newcommand{\solnest}{\widehat{\soln}}

\newcommand{\tr}{^{\top}}

\newcommand{\cb}[1]{\left\{ {#1} \right\}}
\newcommand{\br}[1]{\left( {#1} \right)}
\newcommand{\sq}[1]{\left[ {#1} \right]}
\newcommand{\nrm}[1]{\Vert {#1} \Vert_2}

\newcommand{\order}[1]{O \br{ #1 }}

\newcommand{\R}{{\mathbb R}}

\newcommand{\Xt}{{\mathbf{X}}}
\renewcommand{\u}{\mathbf{u}}

\newcommand{\y}{Y}

\newcommand{\wt}{{\mathbf{w}}}

\DeclareMathOperator*{\argmin}{argmin}

\newcommand{\At}{{\mathbf{A}}}

\newcommand{\Dt}{{\mathbf{D}}}
\newcommand{\Gt}{{\mathbf{G}}}
\newcommand{\Ht}{{\mathbf{H}}}

\newcommand{\It}{{\mathbf{I}}}

\newcommand{\Kt}{{\mathbf{K}}}

\newcommand{\Mt}{{\mathbf{M}}}

\newcommand{\St}{{\mathbf{S}}}

\newcommand{\Ut}{{\mathbf{U}}}
\newcommand{\Vt}{{\mathbf{V}}}
\newcommand{\Wt}{{\mathbf{W}}}

\newcommand{\boldalpha}{\boldsymbol{\alpha}}
\newcommand{\boldbeta}{\boldsymbol{\beta}}

\newcommand{\boldSigma}{\boldsymbol{\Sigma}}

\newtheorem{thm}{Theorem}

\newtheorem{lem}[thm]{Lemma}

\newtheorem{defn}{Definition}

\renewcommand{\algorithmicrequire}{\textbf{Input:}}
\renewcommand{\algorithmicensure}{\textbf{Output:}}

\begin{document}

\twocolumn[

\aistatstitle{\locod: Distributing Statistical Estimation Using Random Projections}

\aistatsauthor{ Christina Heinze \And Brian McWilliams \And Nicolai Meinshausen}

\aistatsaddress{ Seminar for Statistics, ETH Z\"urich \And Disney Research \And Seminar for Statistics, ETH Z\"urich } ]

\begin{abstract}
We present \locod, a communication-efficient algorithm for distributed
statistical estimation. \locod assumes that the data is distributed across workers according to the features rather than the
samples. It requires only a single round of communication where
low-dimensional random projections are used to approximate the
dependencies between features available to different workers. We show that \locod has bounded approximation error which only depends weakly on the number of workers. We compare \locod against a state-of-the-art distributed optimization method on a variety of real world datasets and show that it obtains better speedups while retaining good accuracy. In particular, \locod 
allows for fast cross validation as only part of the algorithm depends on the regularization parameter.
\end{abstract}

\section{Introduction}
Many statistical estimation tasks amount to solving an optimization problem of the form
\begin{equation}
\min_{\soln\in\R^{\dims}} J(\soln):= \sum_{i=1}^\samp f_i(\soln \tr \x_i) + \frac{\lambda}{2}\nrm{\soln}^2
\label{eq:primal}
\end{equation}
where $\lambda>0$ is the regularization parameter. The loss functions $f_i(\soln\tr \x_i)$ depend on labels $y_i\in\R$ and linearly on the coefficients, $\soln$ through a vector of covariates, $\x_i\in\R^{\dims}$. Furthermore, we assume all $f_i$ to be convex and smooth with Lipschitz continuous gradients. Concretely, when $f_i(\soln \tr \x_i) = (y_i - \soln\tr\x_i)^2$, Eq.~\eqref{eq:primal} corresponds to ridge regression; for logistic regression $f_i(\soln \tr \x_i) = \log{(1 + \exp{(-y_i\soln\tr\x_i)})}$.

For large-scale problems, it is no longer practical to solve even relatively simple estimation tasks such as \eqref{eq:primal} on a single machine. To deal with this, approaches to distributed data analysis have been proposed that take advantage of many cores or computing nodes on a cluster. A common idea which links many of these methods is stochastic optimization. Typically, each of the workers only sees a small portion of the data points and performs incremental updates to a global parameter vector. It is typically assumed that the number of data points, $\samp$, is very large compared with the number of features, $\dims$, or that the data is extremely sparse. In such settings -- which are common, but not ubiquitous in large datasets -- distributed stochastic optimization algorithms perform well but may converge slowly otherwise.

A fundamentally different approach to distributing learning is for each worker to only have access to a portion of the available features. Distributing according to the features could be a preferable alternative for several reasons. 
Firstly, for {\bf high-dimensional data}, where $\dims$ is large relative to $\samp$, better scaling can be achieved. This setting is challenging, however, since most loss functions are not separable across coordinates. 
High-dimensional data is commonly encountered in the fields of bioinformatics, climate science and computer vision. Furthermore, for a variety of prediction tasks it is often beneficial to map input vectors into a higher dimensional feature space, e.g.\ using deep representation learning or considering higher-order interactions.
Secondly, {\bf privacy}. Individual blocks of features could correspond to sensitive information (such as medical records) which should be included in the predictive model but is not allowed to be communicated in an un-disguised form. 

\paragraph{Our contribution.}
In this work we introduce \locod to solve problems of the form \eqref{eq:primal} in the distributed setting when each worker only has access to a subset of the \emph{features}. \locod is an extension of the \loco algorithm \cite{mcwilliams2014loco} which was recently proposed for solving distributed ridge regression in this setting. We propose an alternative formulation where each worker instead locally solves a \emph{dual} optimization problem. \locod has a number of practical and theoretical improvements over the original algorithm:
\begin{itemize}
\item \locod is applicable to a wider variety of smooth, convex $\ell_2$ penalized loss minimization problems encompassing many widely used regression and classification loss functions, including ridge regression, logistic regression and others. 
\item In \S\ref{sec:analysis} we provide a more intuitive and tighter theoretical result which crucially does not depend on specific details of the ridge regression model and has weaker dependence on the number of workers, $\blocks$. 
\item We also show that \emph{adding} (rather than concatenating) random features allows for an efficient implementation yet retains good approximation guarantees.
\end{itemize}
In \S\ref{sec:implementation} we report experimental results with high-dimensional real world  datasets corresponding to two different problem domains: climate science and computer vision. We compare \locod with \cocoap, a recently proposed state-of-the-art algorithm for distributed dual coordinate ascent \cite{ma2015adding}. Our experiments show that \locod demonstrates better scaling with $\blocks$ than \cocoap while retaining a good approximation of the optimal solution. We provide an implementation of \locod in Apache Spark\footnote{\url{http://spark.apache.org/}}. The portability of this framework ensures that \locod is able to be run in a variety of distributed computing environments. 

\vspace{-0.15cm}
\section{Related work} \vspace{-0.1cm}
\subsection{Distributed Estimation}
Recently, several asynchronous stochastic gradient descent (SGD) methods \cite{hogwild,Duchi:2013} have been proposed for solving problems of the form \eqref{eq:primal} in a parallel fashion in a multi-core, shared-memory environment and have been extended to the distributed setting. For such methods, large speedups are possible with asynchronous updates when the data is sparse.  However, in some problem domains the data collected is dense with many correlated features. Furthermore, the $\dims\gg\samp$ setting can result in slow convergence. In the distributed setting, such methods can be impractical since the cost of communicating updates can dominate other computational considerations. 

Jaggi et al. proposed a communication-efficient distributed dual coordinate ascent algorithm (\cocoa resp.\ \cocoap) \cite{jaggi2014communication, ma2015adding}. Each worker makes multiple updates to its local dual variables before communicating the corresponding primal update. This allows for trading off communication and convergence speed. Notably they show that convergence is actually independent of the number of workers, thus \cocoap exhibits \emph{strong scaling} with $\blocks$.

Other recent work considers solving statistical estimation tasks using a single round of communication \cite{zhang:2013, liu2014distributed}. However, all of these methods consider only distributing over the rows of the data where an i.i.d. assumption on the observations holds.

On the other hand, few approaches have considered distributing across the columns (features) of the data. This is a more challenging task for both estimation and optimization since the columns are typically assumed to have arbitrary dependencies and most commonly used loss functions are not separable over the features. Recently, \loco was proposed to solve ridge regression when the data is distributed across the features  \cite{mcwilliams2014loco}. \loco requires a single round to communicate small matrices of randomly projected features which approximate the dependencies in the rest of the dataset (cf.\ Figure~\ref{fig:loco_partition}). Each worker then optimizes its own sub-problem independently and finally sends its portion of the solution vector back to the master where they are combined.  
\loco makes no assumptions about the correlation structure between features. It is therefore able to perform well in challenging settings where the features are correlated between blocks and is particularly suited when $\dims\gg \samp$. Indeed, since the relative dimensionality of local problems decreases when splitting by columns, they are easier in a statistical sense. \loco makes no assumptions about data sparsity so it is also able to obtain speedups when the data is dense. 

One-shot communication schemes are beneficial as the cost of communication consists of a fixed cost and a cost that is proportional to the size of the message. Therefore, it is generally cheaper to communicate a few large objects than many small objects. 

\subsection{Random projections for estimation and optimization}
\label{sec:rps}

\begin{figure}[!tp]
\vspace{-5pt}
\includegraphics[trim=0 15 10 10, clip, width=0.45\textwidth, keepaspectratio=true]{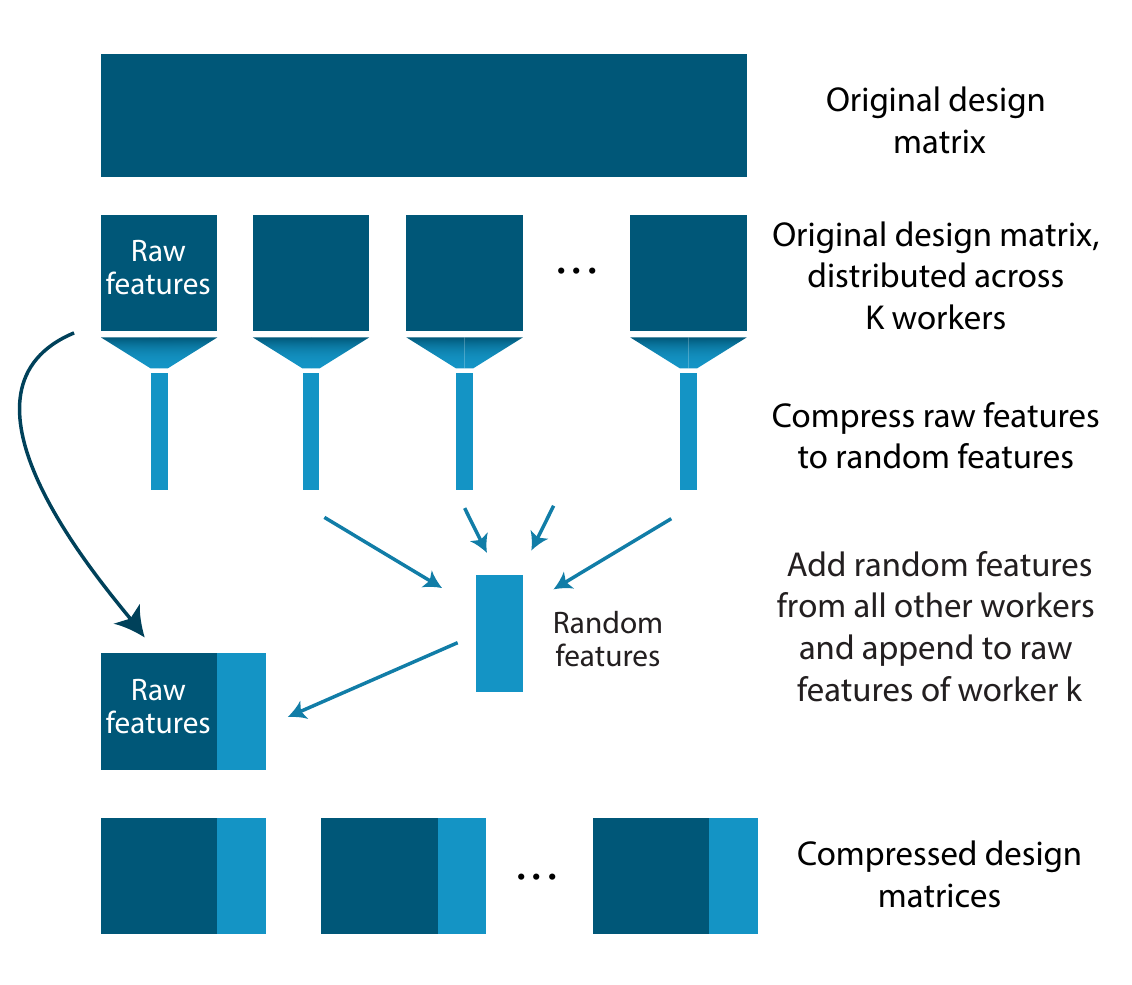}
\caption{\small Schematic for the distributed approximation of a large data set with random projections, used by \locod. \label{fig:loco_partition}}
\end{figure}

Random projections are low-dimensional embeddings $\RP :\R^\dimsk\rightarrow \R^{\dimsks}$ which approximately preserve an entire subspace of vectors.  
They have been extensively used to construct efficient algorithms when the sample-size is large in a variety of domains such as: nearest neighbours \cite{Ailon:2009}, matrix factorization \cite{Boutsidis:2012tv}, least squares \cite{Dhillon:2013wz,mcwilliams2014fast} and recently in the context of optimization \cite{pilanci2015newton}. 

We concentrate on the Subsampled Randomized Hadamard Transform (SRHT), a structured random projection \cite{Tropp:2010uo}. 
The SRHT consists of a projection matrix, $\RP= \sqrt{\dimsk/\dimsks} \Dt\Ht\St$ \cite{Boutsidis:2012tv} with the definitions:
{(\bf i)} $\St\in\R^{\dimsk\times\dimsks}$ is a subsampling matrix. 
{(\bf ii)} $\Dt\in\R^{\dimsk\times\dimsk}$ is a diagonal matrix whose entries are drawn independently
  from $\{-1, 1\}$. 
{(\bf iii)} $\Ht \in \R^{\dimsk \times \dimsk}$ is a normalized Walsh-Hadamard
  matrix. The key benefit of the SRHT is that due to its recursive definition the product between $\RP\tr$ and $\u\in\R^\dimsk$ can be computed in $\order{\dimsk\log\dimsk}$ time while never constructing $\RP$ explicitly. 

For moderately sized problems, random projections have been used to reduce the dimensionality of the data prior to performing regression \cite{kaban:2014,lu:2013}. However after projection, the solution vector is in the compressed space and so interpretability of coefficients is lost. Furthermore, the projection of the low-dimensional solution back to the original high-dimensional space is in fact guaranteed to be a \emph{bad} approximation of the optimum \cite{zhang2012recovering}.  

\paragraph{Dual Random Projections.}
Recently, \cite{zhang2012recovering,zhang2014random} studied the effect of random projections on the \emph{dual} optimization problem. 
For the primal problem in Eq.~\eqref{eq:primal}, defining $\Kt = \Xt\Xt\tr$, we have the corresponding dual
\begin{equation}
\max_{\boldalpha\in\R^\samp} -  \sum_{i=1}^\samp f_i^*(\alpha_i) - \frac{1}{2\samp\lambda} \boldalpha \tr \Kt \boldalpha  \label{eq:alphaopt}
\end{equation}
where $f^*$ is the conjugate Fenchel dual of $f$ and $\lambda > 0$. For example, for squared loss functions $f_i(u) = \frac{1}{2}(y_i-u)^2$, we have $f_i^*(\alpha) = \frac{1}{2} \alpha^2 + \alpha y_i$.
For problems of this form, the dual variables can be directly mapped to the primal variables, such that for a vector $\alphaopt$ which attains the maximum of \eqref{eq:alphaopt},  the optimal primal solution has the form
$
\solnopt(\alphaopt) = -\frac{1}{ \samp \lambda}\Xt\tr \alphaopt . 
$

Clearly, a similar dual problem to \eqref{eq:alphaopt} can be defined in the projected space. Defining $\tilde{\Kt} = (\Xt\RP)(\Xt\RP)\tr$ we have
\begin{equation}
\max_{\boldalpha\in\R^\samp} -\sum_{i=1}^\samp f_i^*(\alpha_i) - \frac{1}{2\samp\lambda} \boldalpha \tr \tilde{\Kt} \boldalpha \label{eq:alphatil} .
\end{equation}

Importantly, the vector of dual variables does not change dimension depending on whether the original problem \eqref{eq:alphaopt} or the projected problem \eqref{eq:alphatil} is being solved. 
Under mild assumptions on the loss function, by mapping the solution to this new problem, $\alphatil$, back to the original space one obtains a vector 
$
\tilde{\soln}(\alphatil) = -\frac{1}{\samp \lambda}\Xt\tr \alphatil
$
, which is a \emph{good} approximation to $\solnopt$, the solution to the original problem \eqref{eq:primal} \cite{zhang2012recovering,zhang2014random}.

\section{The \locod algorithm}

\vspace{-.5cm}
\begin{minipage}{0.475\textwidth}

\begin{algorithm}[H]
\caption{\locod  \label{alg:batch}}
	\algorithmicrequire\; Data: $\Xt$, $\y$, no. workers: $\blocks$ \\ Parameters: $\dimsks$, $\lambda$
  \begin{algorithmic}[1]
  \STATE Partition $\{p\}$ into $\blocks$ subsets of equal size $\dimsk$ and distribute feature vectors in $\Xt$ accordingly over $\blocks$ workers.
    \FOR{\textbf{each} worker $k\in\{1,\ldots\blocks\}$ \textbf{in parallel}}
        \STATE Compute and send random features $ \Xt_k \RP_k$.
        \STATE Receive random features and construct $\feats_k$. 
        \STATE $\alphatil_k \leftarrow \texttt{LocalDualSolver}(\feats_k,\y,\lambda) $
        \STATE $\solnest_k = -\frac{1}{\samp \lambda}\Xt_k\tr\alphatil_k$
        \STATE Send $ \solnest_k $ to driver.
    \ENDFOR
	  \end{algorithmic}
  \algorithmicensure\; Solution vector: $\solnest = \sq{\solnest_1,\ldots,\solnest_\blocks}$
\end{algorithm} 
\end{minipage}

In this section we detail the \locod algorithm. \locod differs from the original \loco algorithm in two important ways. {\bf (i)} The random features from each worker are summed, rather than concatenated, to obtain a $\dimsks$ dimensional approximation allowing for an efficient implementation in a large-scale distributed environment.  
{\bf (ii)} Each worker solves a local \emph{dual} problem similar to \eqref{eq:alphatil}. This allows us to extend the theoretical guarantees to a larger class of estimation problems beyond ridge regression (\S\ref{sec:analysis}). 

We consider the case where $\dims$ features are distributed across $K$ different workers in non-overlapping subsets $\dimset_{1},\ldots,\dimset_{\blocks}$ of equal size\footnote{This is for simplicity of notation only, in general the partitions can be of different sizes.}, $\dimsk=\dims/\blocks$.  

Since most loss functions of interest are not separable across coordinates, a key challenge addressed by \locod is to define a local minimization problem for each worker to solve \emph{independently} and \emph{asynchronously} while still maintaining important dependencies between features in different blocks and keeping communication overhead low. Algorithm \ref{alg:batch} details \locod in full.

We can rewrite \eqref{eq:primal} making explicit the contribution from block $k$. 
Letting $\Xt_{k}\in\R^{\samp\times \dimsk}$ be the sub-matrix whose columns correspond to the coordinates in $\dimset_{k}$ (the ``raw'' features of block $k$) and $\Xt_{\minusk}\in\R^{\samp\times (\dims-\dimsk)}$ be the remaining columns of $\Xt$, we have
\begin{align} \label{eq:optim_global}
J(\soln) = &  \sum_{i=1}^\samp f_i\br{\x_{i,k} \tr \soln_{\text{raw}} + \x_{ i, \minusk} \tr \soln_{\minusk}} 
\nonumber\\
& 
+ \lambda  \big(\nrm{ \soln_{\text{raw}}}^2 +  \nrm{\soln_{\minusk}}^2\big).
 \end{align}
 Where $\x_{i,k}$ and $\x_{i,(-k)}$ are the rows of $\Xt_{k}$ and $\Xt_{\minusk}$ respectively.
We replace $\Xt_{\minusk}$ in each block with a low-dimensional randomized approximation which preserves its contribution to the loss function. This procedure is described in Figure \ref{fig:loco_partition}.

In {\bf Step 5}, these matrices of random features are communicated and worker $k$ constructs the matrix
\begin{equation}
\feats_k\in\R^{\samp\times(\dimsk+ \dimsks)} = \sq{\Xt_k, \sum_{k'\neq k}  \Xt_k \RP_k},
\label{eq:concat}
\end{equation}
which is the concatenation of worker $k$'s raw features and the \emph{sum} of the random features from all other workers. $\RP$ is the SRHT matrix introduced in \S\ref{sec:rps}. 

As we prove in Lemma~\ref{lem:sumRF}, summing $\R^\dimsk\rightarrow\R^{\dimsks}$-dimensional random projections from $(\blocks-1)$ blocks is equivalent to computing the $\R^{(\dims - \dimsk)}\rightarrow\R^{\dimsks}$-dimensional random projection in one go. The latter operation is impractical for very large $\dims$ and not applicable when the features are distributed. Therefore, summing the random features from each worker allows the dimensionality reduction to be distributed across workers.
Additionally, the summed random feature representation can be computed and combined very efficiently. 
We elaborate on this aspect in \S\ref{sec:implementation}.

For a single worker the local, approximate primal problem is then
\begin{equation} \label{eq:local_primal}
\min_{\solnRP\in\R^{\dimsk+\dimsks}} J_k(\solnRP) := \sum_{i=1}^\samp f_i(\solnRP \tr \bar{\x}_i) + \frac{\lambda}{2}\nrm{\solnRP}^2
\end{equation}
where $\bar{\x}_i \in \R^{\dimsk + \dimsks}$ is the $i^{th}$ row of $\feats_k$. The corresponding dual problem for each worker in the \locod algorithm is
\begin{equation}\label{eq:localdual}
\max_{\boldalpha\in\R^{\samp}} - \sum_{i=1}^\samp f^{*}_i(\alpha_i) - \frac{1}{2\samp\lambda}  \boldalpha\tr \tilde{\Kt}_k \boldalpha, \quad \tilde{\Kt}_k = \feats_k\feats_k\tr .
\end{equation}

The following steps in Algorithm \ref{alg:batch} detail respectively how the solution to \eqref{eq:localdual} and the final \locod estimates are obtained.

\paragraph{Step 6. \texttt{LocalDualSolver}.} The \texttt{LocalDualSolver} computes the solution for \eqref{eq:localdual}, the local dual problem.
The solver can be chosen to best suit the problem at hand. This will depend on the absolute size of $\samp$ and $\dimsk + \dimsks$ as well as on their ratio. For example, we could use SDCA \cite{shalev2013stochastic} or Algorithm 1 from \cite{zhang2012recovering}.

\paragraph{Step 7. Obtaining the global primal solution.} Each worker maps its local dual solution to the primal solution corresponding only to the coordinates in $\dimset_k$. In this way, each worker returns coefficients corresponding only to its own raw features. The final primal solution vector is obtained by concatenating the $K$ local solutions. Unlike \loco, we no longer require to discard the coefficients corresponding to the random features for each worker. Consequently, computing estimates is more efficient (especially when $\dims \gg \samp$).

\vspace{-0.25cm}
\section{\locod Approximation Error} \label{sec:analysis}
\vspace{-0.15cm}
In this section we bound the recovery error between the \locod solution and the solution to Eq.~\eqref{eq:primal}. 

\begin{thm}[\locod error bound] \label{thm:dual_loco}
Consider a matrix $\Xt\in\R^{\samp\times\dims}$ with rank, $\rank$. Assume that the loss $f(\cdot)$ is smooth and Lipschitz continuous. For a subsampling dimension $\dimsks \geq c_1 \dims \blocks $ where $0\leq c_1 \leq 1/\blocks^2$, let $\solnopt$ be the solution to \eqref{eq:primal} and $\solnest$ be the estimate returned by Algorithm \ref{alg:batch}. We have with probability at least $1-\blocks \br{\delta + \frac{\dims-\dimsk}{e^\rank}}$ 
\begin{eqnarray}
\nrm{\solnest - \solnopt} &\leq& \frac{  \varepsilon  }{1-\varepsilon} \nrm{\solnopt} ~, \label{eq:final_bound}\\
&\text{where}& \quad \varepsilon = \sqrt{\frac{c_0 \log(2\rank/\delta) \rank}{c_1 \dims}} < 1. \nonumber
\end{eqnarray}
\end{thm}

\begin{proof}
By Lemma \ref{thm:lowrank} and applying a union bound we can decompose the global optimization error in terms of the error due to each worker as
$
\nrm{\solnopt - \solnest} = \sqrt{\sum_{k=1}^K \nrm{\solnopt_k - \solnest_k}^2}
 \leq \sqrt{\blocks} \frac{\rho}{1-\rho} \nrm{\solnopt}, 
$
which holds with probability $1-\blocks \br{\delta + \frac{\dims-\dimsk}{e^\rank}}$. The final bound, \eqref{eq:final_bound} follows by setting  $\rho = \sqrt{\frac{c_0 \log(2\rank/\delta) \rank}{\dimsks}}$ and $\dimsks \geq c_1 \dims \blocks $ and noting that $\sqrt{\blocks}\cdot \frac{ \frac{\varepsilon}{\sqrt{\blocks}}} {1- \frac{\varepsilon}{\sqrt{\blocks}}} \leq \frac{\varepsilon}{1-\varepsilon}$. 
\end{proof}

Theorem \ref{thm:dual_loco} guarantees that the solution to \locod will be close to the optimal solution obtained by a single worker with access to all of the data. Our result relies on the data having rank $\rank \ll \dims$. In practice, this assumption is often fulfilled, in particular when the data is high dimensional. 
For a large enough projection dimension, the bound has only a weak dependence on $\blocks$ through the union bound used to determine $\xi$. The error is then mainly determined by the ratio between the rank and the random projection dimension. When the rank of $\Xt$ increases for a fixed $\dims$, we need a larger projection dimension to accurately capture its spectrum. On the other hand, the failure probability increases with $\dims$ and decreases with $\rank$. However, this countering effect is negligible as typically $\log{(\dims - \dimsk)} \ll \rank$.

\vspace{-0.25cm}
\section{Implementation and Experiments} \label{sec:implementation}
\vspace{-0.15cm}
In this section we report on the empirical performance of \locod in two sets of experiments. The first demonstrates the performance of \locod in a large, distributed classification task. The second is an application of $\ell_2$ penalized regression to a problem in climate science where accurate recovery of the coefficient estimates is of primary interest.

{\bf Cross validation.} 
In most practical cases, the regularization parameter $\lambda$ is unknown and has to be determined via $v$-fold cross validation (CV).
The chosen algorithm is usually run entirely once for each fold and each of $\lenlambda$ values of $\lambda$, leading to a runtime that is approximately $v \cdot \lenlambda$ as large as the runtime of a single run\footnote{``Approximately'' since the cross validation procedure also requires time for testing. For a single run we only count the time it takes to estimate the parameters.}. 
In this context, \locod has the advantage that steps 3 and 4 in Algorithm \ref{alg:batch} are independent of $\lambda$. Therefore, these steps only need to be performed \textit{once per fold}. In step 5, we then estimate $\alphatil_k$ for each value in the provided sequence for $\lambda$. Thus, the runtime of \locod will increase by much less than $v \cdot \lenlambda$ compared to the runtime of a single run. The performance of each value for $\lambda$ is then not only averaged over the random split of the training data set into $v$ parts but also over the randomness introduced by the random projections which are computed and communicated once per fold. The procedure is provided in full detail in Algorithm~\ref{alg:cv} in Appendix~\ref{sec:supp_exp}.

{\bf Implementation details.} 
We implemented \locod using the Apache Spark framework\footnote{A software package will be made available under the Apache license.}. Spark is increasingly gaining traction in the research community as well as in industry due to its easy-to-use high-level API and the benefits of in-memory processing. Spark is up to 100$\times$ faster than Hadoop MapReduce.
Additionally, Spark can be used in many different large-scale computing environments and the various, easily-integrated libraries for a diverse set of tasks greatly facilitate the development of applications.  

\begin{figure}[!tp]
\center
\vspace{-7.5pt}
\subfloat[\xspace]{
    \includegraphics[width=0.125\textwidth, keepaspectratio=true]{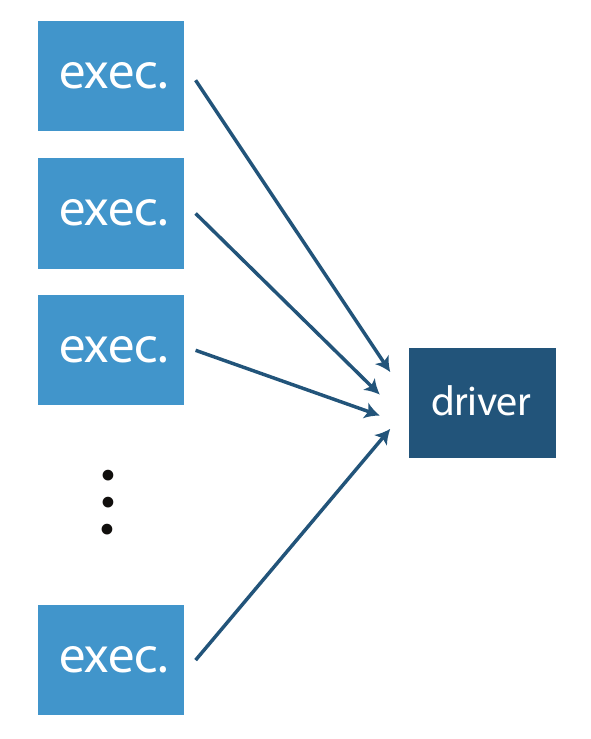}
    \label{fig:notreeagg}
}
\hspace{1cm}
\subfloat[\xspace]{
    \includegraphics[width=0.205\textwidth, keepaspectratio=true]{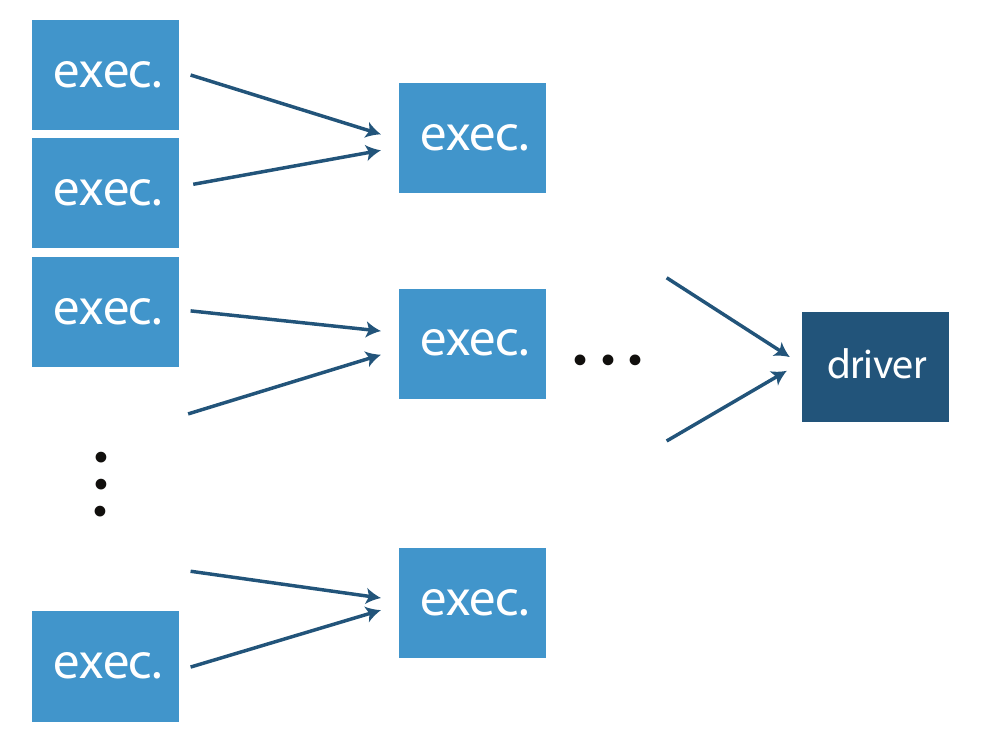}
    \label{fig:treeagg}
}
\vspace{-7.5pt}
\caption{{\small Schematic for the aggregation of the random features in Spark. \protect\subref{fig:notreeagg} When concatenating the random features naively, every worker node (exec.) sends its random features to the driver from where they are broadcasted to all workers. \protect\subref{fig:treeagg} Using the \texttt{treeReduce} scheme we can reduce the load on the driver by summing the random features from each worker node as this operation is associative and commutative. Worker $k$ is only required to subtract its own random features locally.} \label{fig:aggreg_comparison}}
\vspace{-0.25cm}
\end{figure}

When communicating and summing the random features in Spark, \locod leverages the \texttt{treeReduce} scheme as illustrated in Figure~\ref{fig:aggreg_comparison}\subref{fig:treeagg}. 
Summing has the advantage that increasing the number of workers simply introduces more layers in the tree structure (Figure~\ref{fig:treeagg}) while the load on the driver remains constant and the aggregation operation also benefits from a parallel execution. Thus, when increasing $\blocks$
 only relatively little additional communication cost is introduced which leads to speedups as demonstrated below. 
 
 In practice, we used the discrete cosine transform (DCT) provided in the FFT library jTransforms\footnote{\url{https://sites.google.com/site/piotrwendykier/software/jtransforms}}\footnote{For the Hadamard transform, $\dimsk$ must be a power of two. For the DCT there is no restriction on $\dimsk$ and very similar theoretical guarantees hold.} and we ran \locod as well as \cocoap on a high-performance cluster\footnote{\cocoap is also implemented in Spark with code available from \url{https://github.com/gingsmith/cocoa}.}.

{\bf Competing methods.} 
For the classification example, the loss function is the hinge loss. Although the problem is non-smooth, and therefore not covered by our theory, we still obtain good results suggesting that Theorem \ref{thm:dual_loco} can be generalized to non-smooth losses. Alternatively, for classification the smoothed hinge or logistic losses could be used. For the regression problem we use the squared error loss and modify \cocoap accordingly. As the \texttt{LocalDualSolver} we use SDCA \cite{shalev2013stochastic}.

We also compared \locod against the reference implementation of distributed loss minimization in the MLlib library in Spark using SGD and L-BFGS. However, after extensive cross-validation over regularization strength (and step size and mini-batch size in case of SGD), we observed that the variance was still very large and so we omit the MLlib implementations from the figures. A comparison between \cocoa and variants of SGD and mini-batch SDCA can be found in \cite{jaggi2014communication}.

{\bf Kaggle Dogs vs Cats dataset.} This is a binary classification task consisting of $25,000$ images of dogs and cats\footnote{\url{https://www.kaggle.com/c/dogs-vs-cats}}. We resize all images to $430\times 430$ pixels and use {\sc Overfeat} \cite{sermanet-iclr-14} -- a pre-trained convolutional neural network -- to extract  $\dims = 200,704$ \textit{fully dense} feature vectors from the $19^{th}$ layer of the network for each image. We train on $\samp_{train}=20,000$ images and test on the remaining $\samp_{test} = 5,000$. The size of the training data is $37$GB with over 4 billion non-zero elements. All results we report in the following are averaged over five repetitions and by ``runtime'' we refer to wall clock time.

Figure~\ref{fig:error} shows the median normalized training and test prediction MSE of \locod and \cocoap for different numbers of workers\footnote{In practice, this choice will depend on the available resources in addition to the size of the data set.}. For \locod, we also vary the size of the random feature representation and choose $\dimsks = \lbrace 0.005, 0.01, 0.02 \rbrace \times (\dims - \dimsk)$. The corresponding errors are labeled with \locod $0.5$, \locod $1$ and \locod $2$.
Note that combinations of $\blocks$ and $\dimsks$ that would result in $\dimsk <\dimsks$ cannot be used, e.g.\ $\blocks = 192$ and $\dimsks = 0.01 \times (\dims - \dimsk)$.
We ran \cocoap until a duality gap of $10^{-2}$ was attained so that the number of iterations varies for different numbers of workers\footnote{For $\blocks$ ranging from $12$ to $192$, the number of iterations needed were $77, 207, 4338, 1966$, resp.\ $3199$.\label{footnote:cocoap_its}}. Notably, for $\blocks = 48$ more iterations were needed than in the other cases which is reflected in the very low training error in this case. The fraction of local points to be processed per round was set to 10\%.
We determined the regularization parameter $\lambda$ via 5-fold cross validation.

While the differences in training errors between \locod and \cocoap are notable, the differences between the test errors are minor as long as the random feature representation is large enough. Choosing $\dimsks$ to be only $0.5$\% of $\dims - \dimsk$ seems to be slightly too small for this data set. When setting $\dimsks$ to be $1$\% of $\dims - \dimsk$ the largest difference between the test errors of \locod and \cocoap is 0.9\%.  The averaged mean squared prediction errors and their standard deviations are collected in Table~\ref{tab:mse_numbers} in Appendix~\ref{sec:supp_exp}. 

\begin{figure}[!tp]
\begin{centering}
\vspace{-7.5pt}
    \includegraphics[trim=0 16 0 0, clip, width=0.3125\textwidth, keepaspectratio=true]{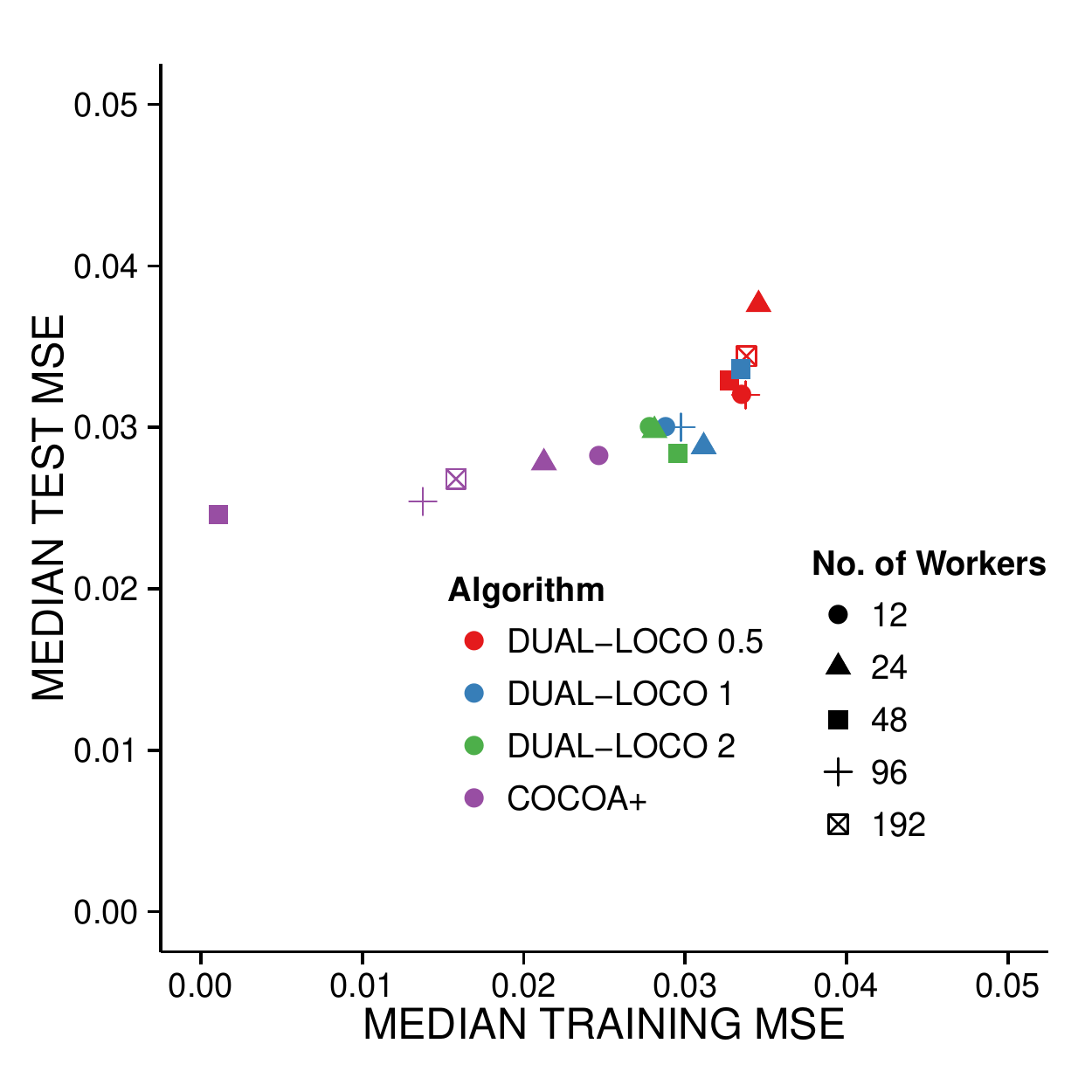}
    \caption{\small Dogs vs Cats data: Median normalized training and test prediction MSE based on 5 repetitions.\label{fig:error}}
\end{centering}
\end{figure}

Next, we would like to compare the wall clock time needed to find the regularization parameter $\lambda$ via 5-fold cross validation. For \cocoap, using the number of iterations needed to attain a duality gap of $10^{-2}$ would lead to runtimes of more than 24 hours for $\blocks \in \{48, 96, 192\}$ when comparing $\lenlambda=20$ possible values for $\lambda$. One might argue that using a duality gap of $10^{-1}$ is sufficient for the cross validation runs which would speed up the model selection procedure significantly as much fewer iterations would be required. Therefore, for $\blocks \geq 48$ we use a duality gap of $10^{-1}$ during cross validation and a duality gap of $10^{-2}$ for learning the parameters, once $\lambda$ has been determined.
 Figure~\ref{fig:totalTime20} shows the runtimes when $\lenlambda=20$ possible values for $\lambda$ are compared; Figure~\ref{fig:cv50}\subref{fig:totalTime50} compares the runtimes when cross validation is performed over $\lenlambda=50$ values. The absolute runtime of \cocoap for a single run is smaller for $\blocks = 12$ and $\blocks = 24$ and larger for $\blocks \in \{48, 96, 192\}$, so using more workers increased the amount of wall clock time necessary for job completion. 
The total runtime, including cross validation and a single run to learn the parameters with the determined value for $\lambda$, is always smaller for \locod, except when $\blocks=12$ and $\lenlambda=20$.

\begin{figure}[!tp]
\begin{centering}
\vspace{-7.5pt}
    \includegraphics[trim=0 8 0 0, clip, width=0.5\textwidth, keepaspectratio=true]{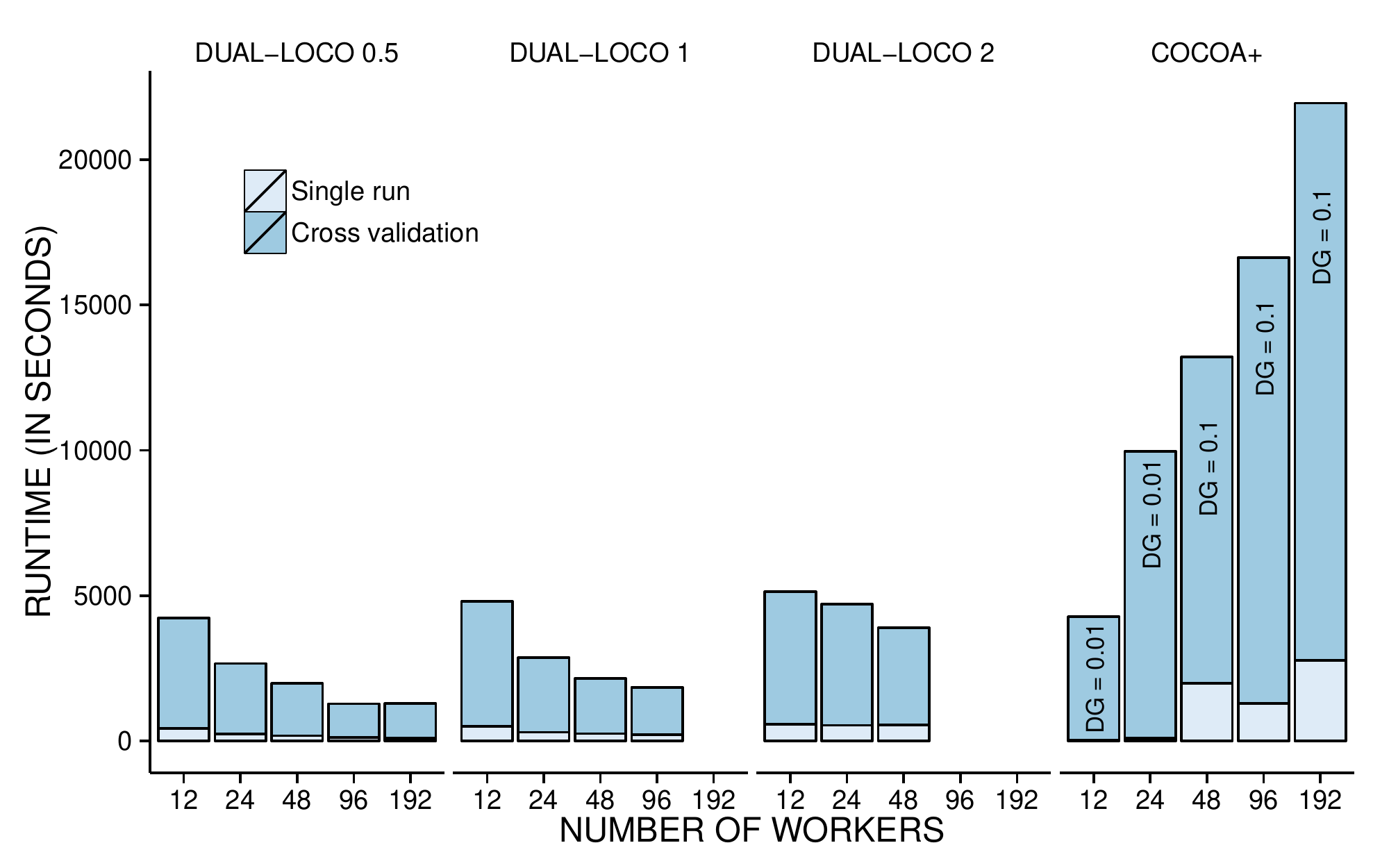}
\caption{\small Total wall clock time including 5-fold CV over $\lenlambda=20$ values for $\lambda$. For \cocoap, we use a duality gap (DG) of $10^{-1}$ for the CV runs when $K \geq 48$. \label{fig:totalTime20}}
\end{centering}
\vspace{-10.5pt}
\end{figure}

Figures~\ref{fig:speedup} and \ref{fig:cv50}\subref{fig:speedup_50} show the relative speedup of \locod and \cocoap when increasing $\blocks$. The speedup is computed by dividing the runtime for $\blocks = 12$ by the runtime achieved for the corresponding $\blocks=\{24,48,96,192\}$. A speedup value smaller than 1 implies an \emph{increase} in runtime. When considering a single run, we run \cocoap in two different settings: {\bf (i)} We use the number of iterations that are needed to obtain a duality gap of $10^{-2}$ which varies for different number of workers\textsuperscript{\ref{footnote:cocoap_its}}. Here, the speedup is smaller than 1 for all $\blocks$.  {\bf (ii)} We fix the number of outer iterations to a constant number. As $\blocks$ increases, the number of inner iterations decreases, making it easier for \cocoap to achieve a speedup. We found that although \cocoap attains a speedup of $1.17$ when increasing $\blocks$ from $12$ to $48$ (equivalent to a decrease in runtime of $14$\%), \cocoap suffers a $24$\% increase in runtime when increasing $\blocks$ from $12$ to $192$. 
  
   For \locod $0.5$ and \locod $1$ we observe significant speedups as $\blocks$ increases. As we split the design matrix by features the number of observations $\samp$ remains constant for different number of workers. At the same time, the dimensionality of each worker's local problem decreases with $\blocks$. 
   Together with the efficient aggregation of the random features, this leads to shorter runtimes. In case of \locod 2, the communication costs dominate the costs of computing the random projection and of the \texttt{LocalDualSolver}, resulting in much smaller speedups. 

Although \cocoap was demonstrated to obtain speedups for low-dimensional data sets \cite{ma2015adding} it is plausible that the same performance cannot be expected on a very high-dimensional data set. This illustrates that in such a high-dimensional setting splitting the design matrix according to the columns instead of the rows is more suitable.

\begin{figure}[!tp]
\begin{centering}
\vspace{-7.5pt}
\hspace{-10mm}
\subfloat[\xspace]{
    \includegraphics[trim=5 10 5 0, clip, width=0.25\textwidth, keepaspectratio=true]{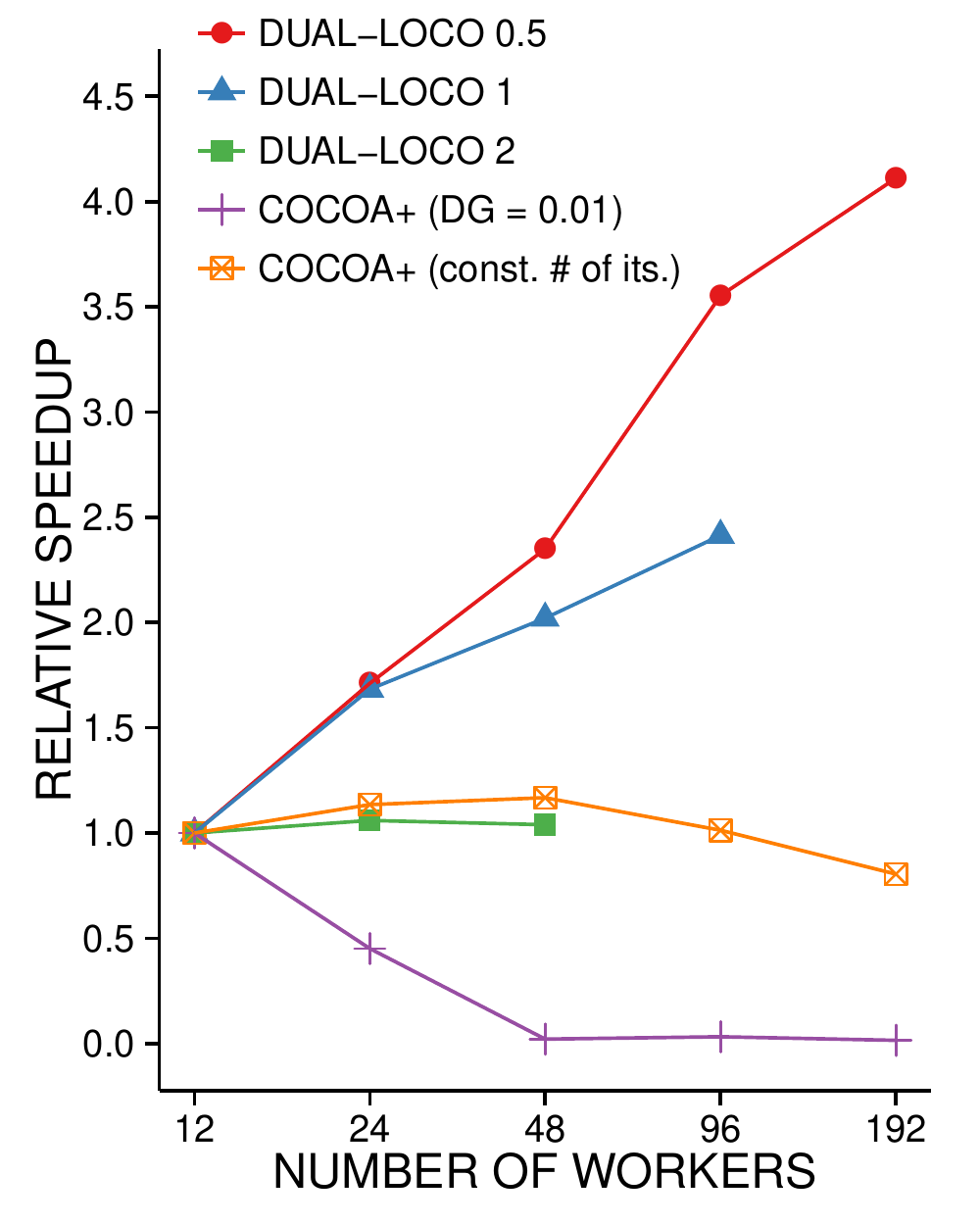}
    \label{fig:speedup_single}
}
\subfloat[\xspace]{
    \includegraphics[trim=5 10 5 0, clip, width=0.25\textwidth, keepaspectratio=true]{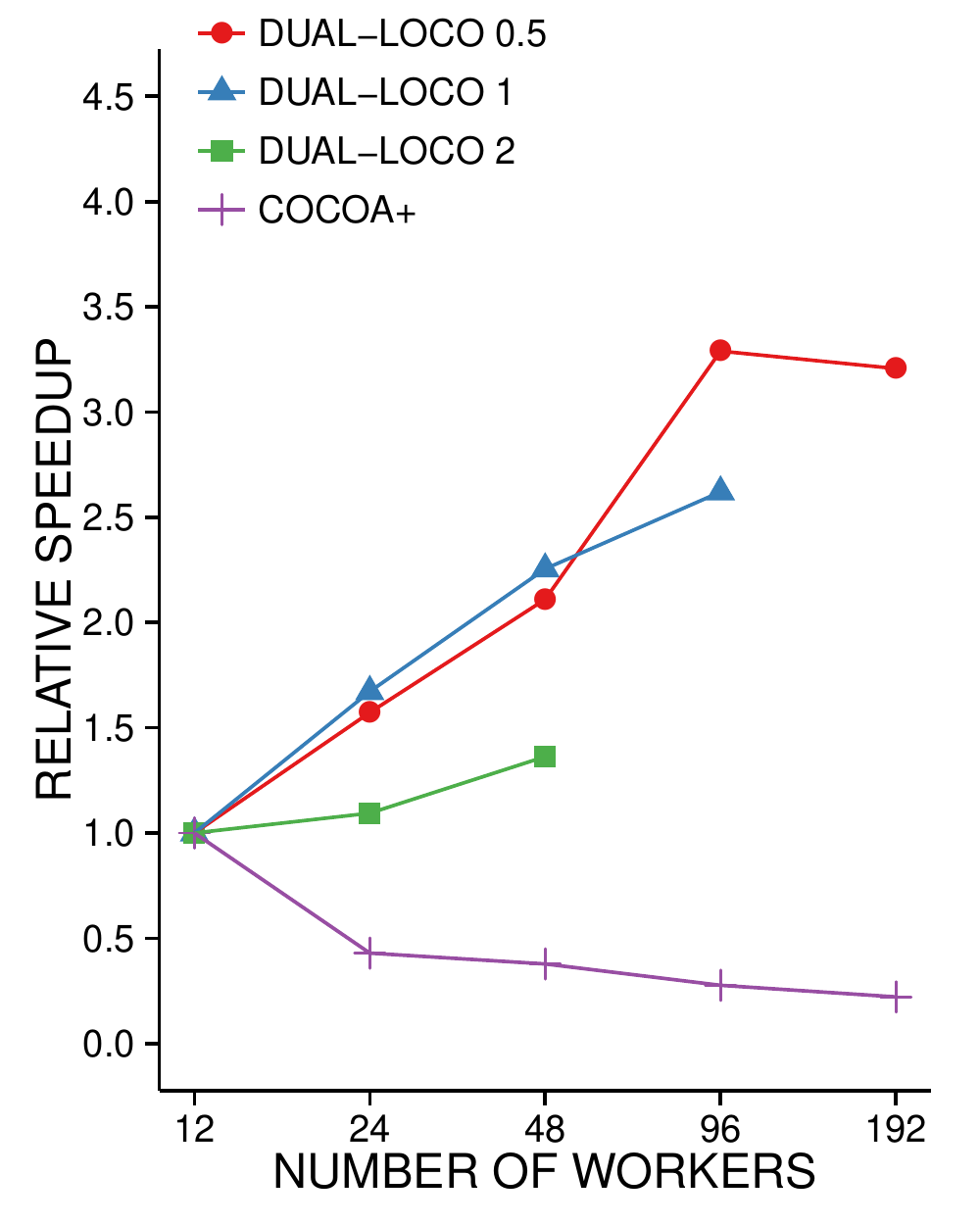}
    \label{fig:speedup_20}
}
\caption{\small Relative speedup for \protect\subref{fig:speedup_single} a single run and \protect\subref{fig:speedup_20} 5-fold CV over $\lenlambda = 20$ values for $\lambda$. \label{fig:speedup}}
\end{centering}
\vspace{-7.5pt}
\end{figure}

\begin{figure}[!tp]
\begin{centering}
\vspace{-7.5pt}
\subfloat[\xspace]{
\hspace{-10mm}
    \includegraphics[trim=0 10 5 15, clip, width=0.28125\textwidth, keepaspectratio=true]{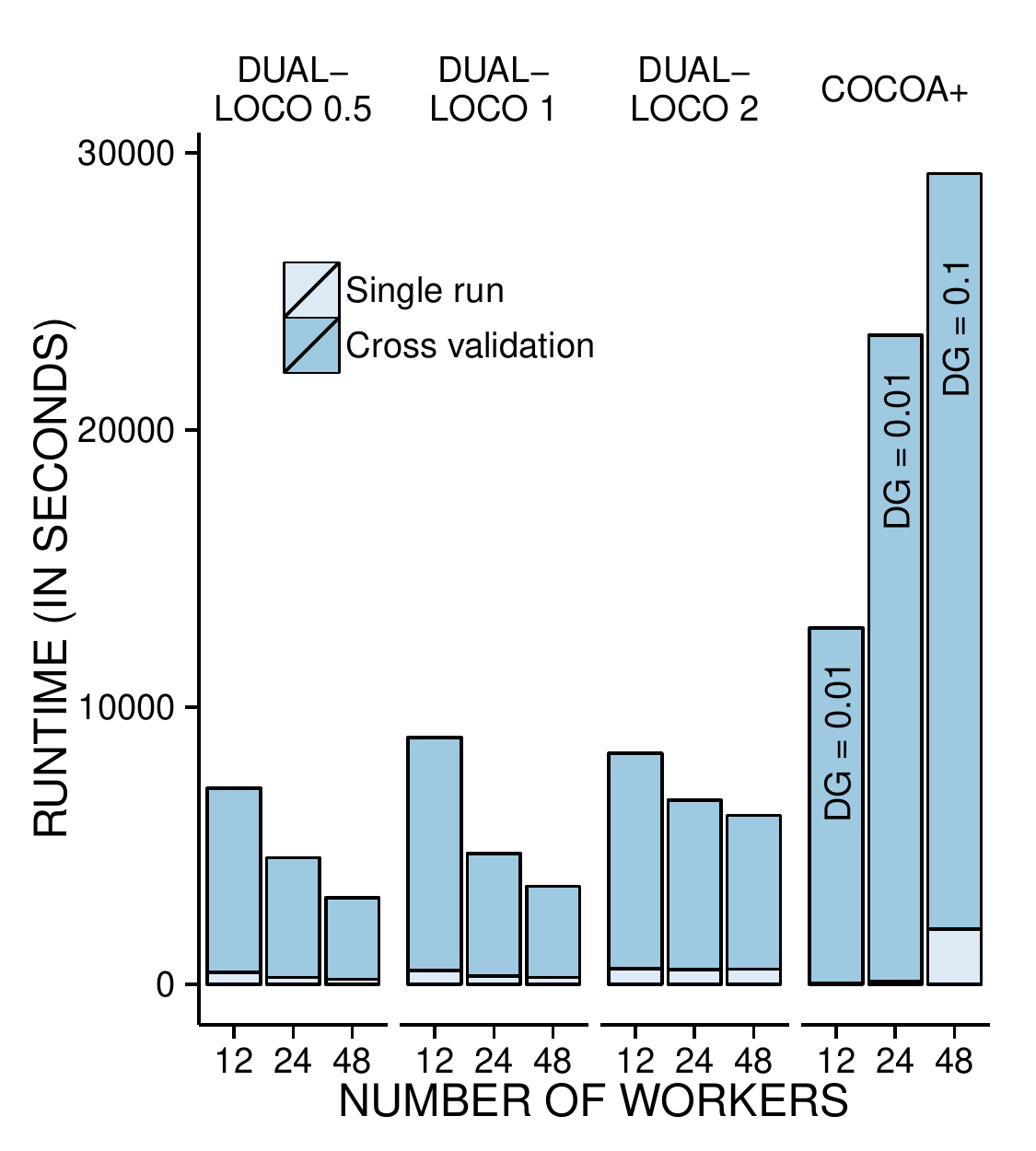}
    \label{fig:totalTime50}
}
\subfloat[\xspace]{
    \includegraphics[trim=5 10 5 0, clip, width=0.25\textwidth, keepaspectratio=true]{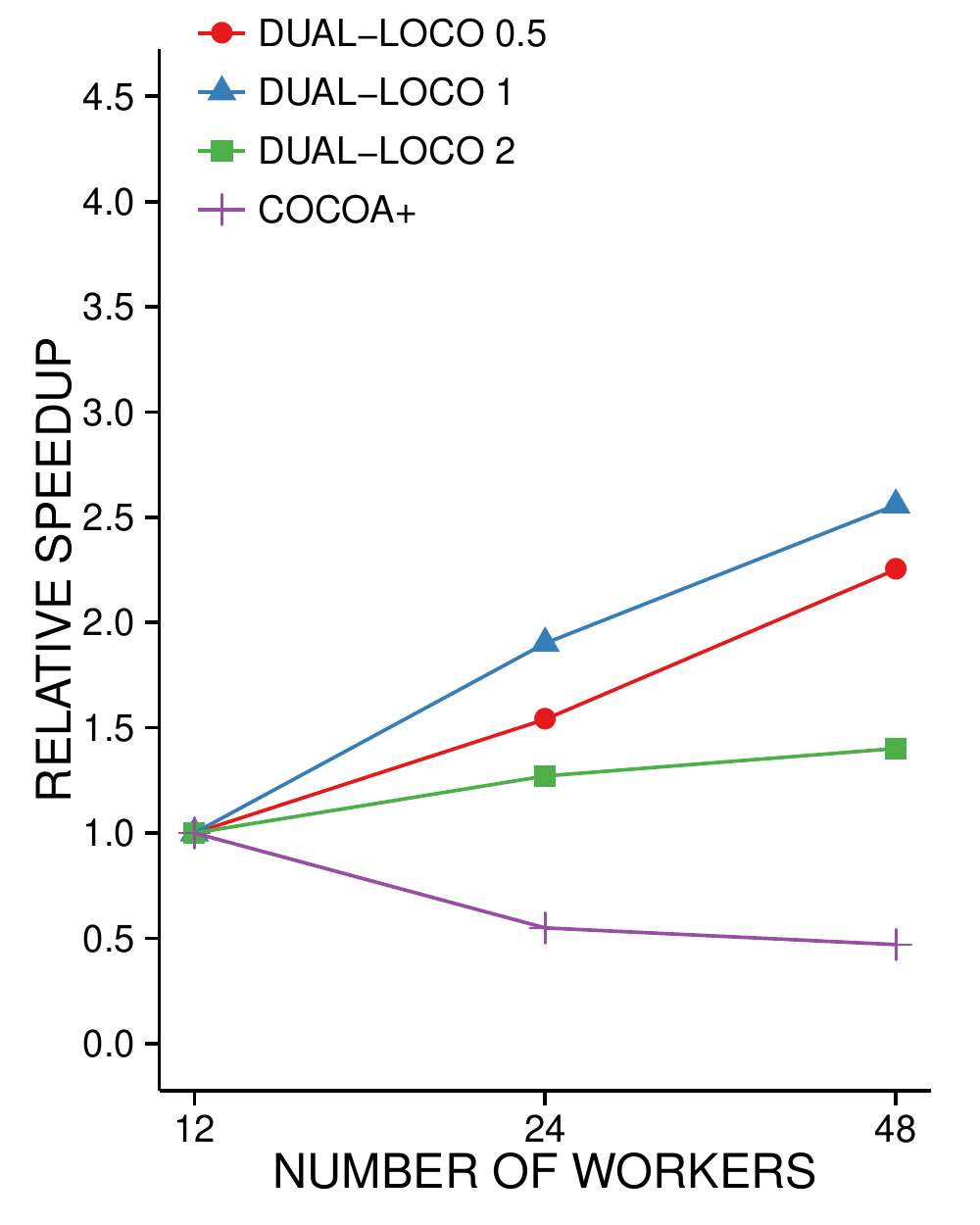}
    \label{fig:speedup_50}
}
\caption{\small 5-fold CV over $\lenlambda = 50$ values for $\lambda$: \protect\subref{fig:totalTime50}
  Total wall clock time and
  \protect\subref{fig:speedup_50} relative speedup. \label{fig:cv50}}
\end{centering}
\vspace{-7.5pt}
\end{figure}

{\bf Climate data.} This is a regression task where we demonstrate that the coefficients returned by \locod are interpretable. The data set contains the outcome of control simulations of the  GISS global circulation model \cite{schmidt2014configuration} and is part of the CMIP5 climate modeling ensemble.  
We aim to forecast the monthly global average temperature $\y$ in February using the air pressure measured in January. Results are very similar for other months. The $\dims = 10,368$ features are pressure measurements taken at $10,368$ geographic grid points in January. The time span of the climate simulation is 531 years and we use the results from two control simulations, yielding $n_{\text{train}} = 849$ and $n_{\text{test}} = 213$. 

In Figure~\ref{fig:beta_comparison} we compare the coefficient estimates for four different methods. The problem is small enough to be solved on a single machine so that the full solution can be computed (using SDCA; cf.\ Figure~\ref{fig:beta_comparison}\subref{fig:beta_full}). This allows us to report the normalized parameter estimation mean squared error ($\text{MSE}_{\solnest}$) with respect to the full solution in addition to the normalized mean squared prediction error ($\text{MSE}$).
The solution indicates that the pressure differential between Siberia (red area, top middle-left) and Europe and the North Atlantic (blue area, top left and top right) is a good predictor for the temperature anomaly. This pattern is concealed in 
Figure~\ref{fig:beta_comparison}\subref{fig:beta_proj} which shows the result of up-projecting the coefficients estimated following a random projection of the columns. Using this scheme for prediction was introduced in \cite{lu:2013}. Although the MSE is similar to the optimal solution, the recovered coefficients are not interpretable as suggested by \cite{zhang2012recovering}. Thus, this method should only be used if prediction is the sole interest. Figure~\ref{fig:beta_comparison}\subref{fig:beta_loco} shows the estimates returned by \locod which is able to recover estimates which are close to the full solution. Finally, Figure~\ref{fig:beta_comparison}\subref{fig:beta_cocoa} shows that \cocoap also attains accurate results.

\begin{figure*}[!tp]
\vspace{-.5cm}
\center
\subfloat[Single machine: Full solution ($\text{MSE} = 0.72$)]{
    \includegraphics[trim=58 73 0 50, clip, width=0.485\textwidth, keepaspectratio=true]{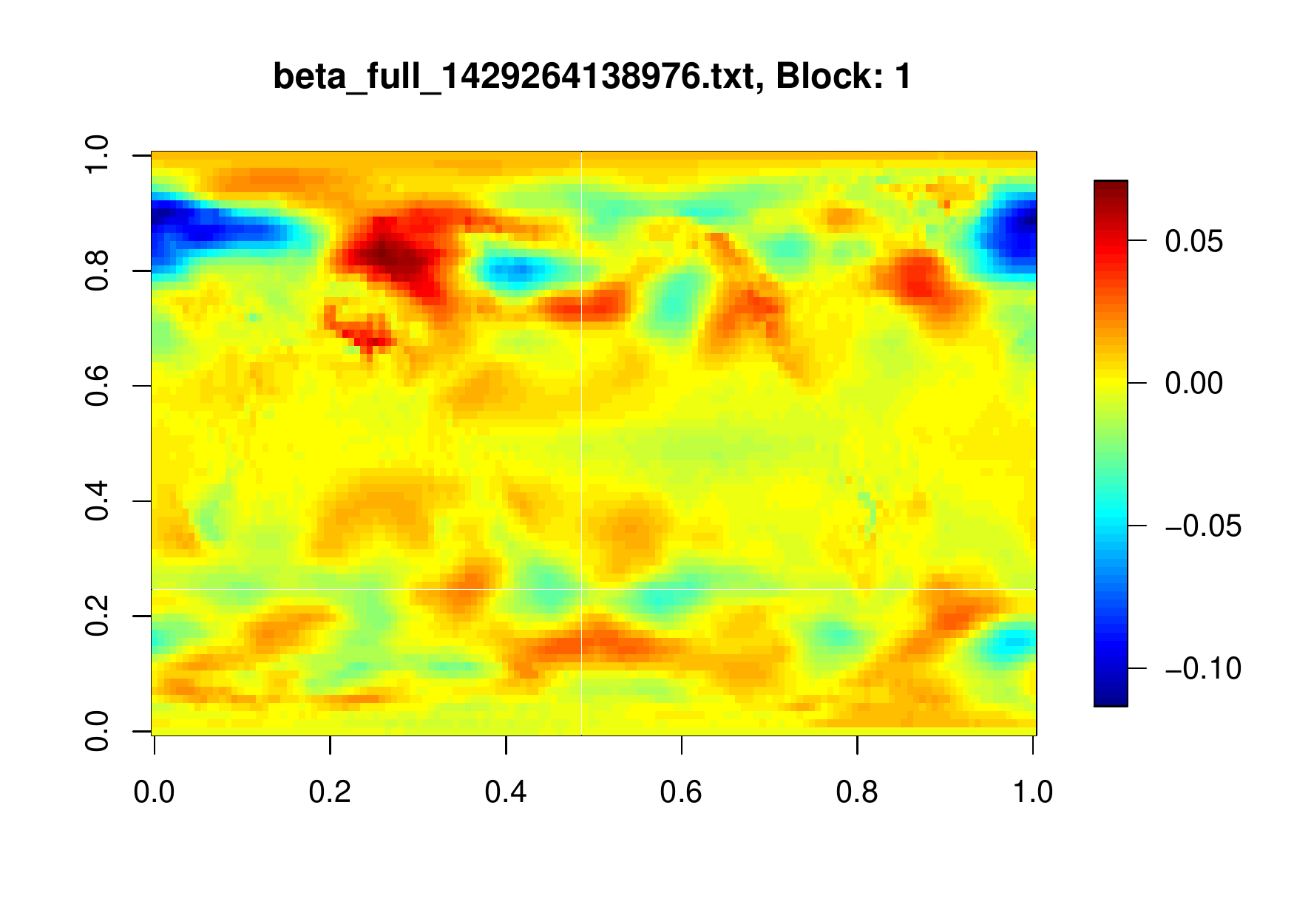}
    \label{fig:beta_full}
}
\subfloat[Single machine: Column-wise compression \newline ($\text{MSE} = 0.73$, $\text{MSE}_{\solnest} = 21.28$)]{
    \includegraphics[trim=58 73 0 50, clip, width=0.485\textwidth, keepaspectratio=true]{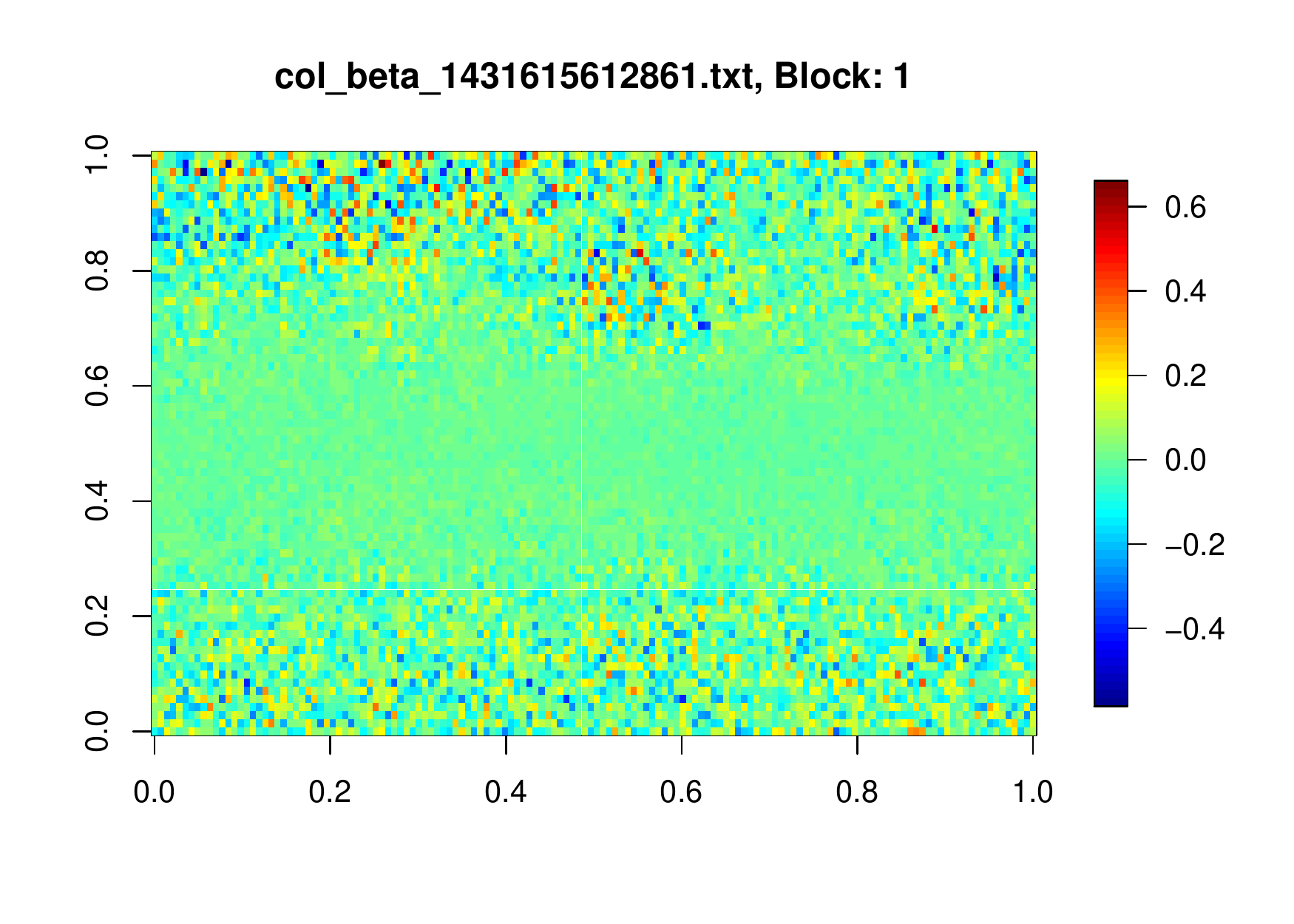}
    \label{fig:beta_proj}
}
\vspace{-0.4cm}
\subfloat[Distributed setting: \locod 10 with $\blocks = 4$ \newline ($\text{MSE} = 0.72$, $\text{MSE}_{\solnest} = 0.02$)]{
    \includegraphics[trim=58 73 0 50, clip, width=0.485\textwidth, keepaspectratio=true]{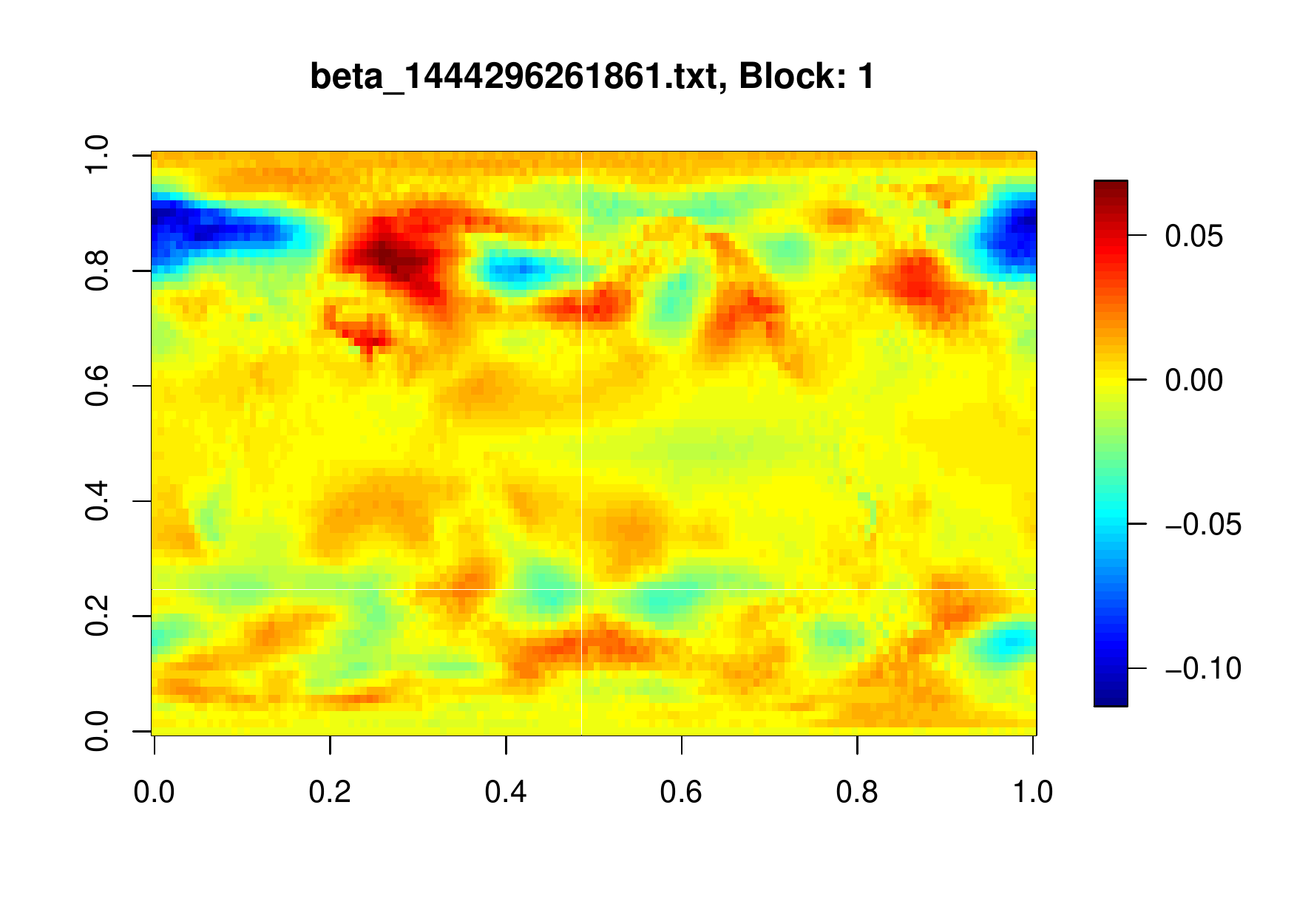}
    \label{fig:beta_loco}
}
\subfloat[Distributed setting: \cocoap  with $\blocks = 4$ \newline ($\text{MSE} = 0.72$, $\text{MSE}_{\solnest} = 0.01$)]{
    \includegraphics[trim=58 73 0 50, clip, width=0.485\textwidth, keepaspectratio=true]{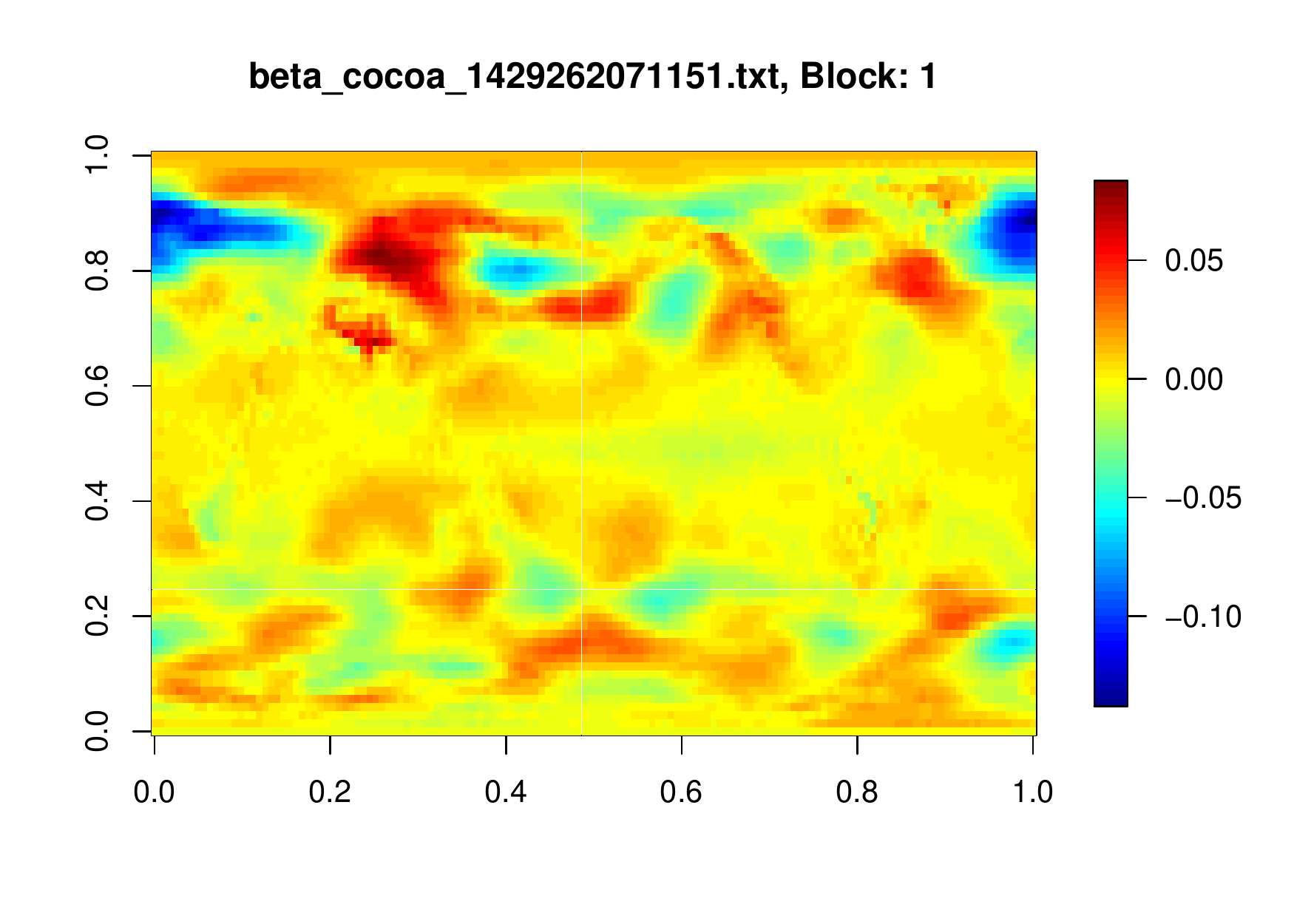}
    \label{fig:beta_cocoa}
}
\vspace{-0.25cm}
\caption{\small Climate data: The regression coefficients are shown as maps with the prime median (passing through London) corresponding to the left and right edge of the plot. The Pacific Ocean lies in the center of each map. 
\label{fig:beta_comparison}}
\vspace{-0.5cm}
\end{figure*}

Considering a longer time period or adding additional model variables such as temperature, precipitation or salinity rapidly increases the dimensionality of the problem while the number of observations remains constant. Each additional variable adds $10,368$ dimensions per month of simulation. Estimating very high-dimensional linear models is a significant challenge in climate science and one where distributing the problem across features instead of observations is advantageous. The computational savings are much larger when distributing across features as $\dims \gg \samp$ and thus reducing $\dims$ is associated with larger gains than when distributing across observations.
\vspace{-0.15cm}
\section{Conclusions and further work}
\vspace{-0.1cm}
We have presented \locod which considers the challenging and rarely studied problem of statistical estimation when data is distributed across features rather than samples.
\locod generalizes \loco to a wider variety of loss functions for regression and classification. We show that the estimated coefficients are close to the optimal coefficients that could be learned by a single worker with access to the entire dataset. The resulting bound is more intuitive and tighter than previous bounds, notably with a very weak dependence on the number of workers. We have demonstrated that \locod is able to recover accurate solutions for large-scale estimation tasks whilst also achieving better scaling than a state-of-the-art competitor, \cocoap, as $\blocks$ increases. Additionally, we have shown that \locod allows for fast model selection using cross-validation.

The dual formulation is convenient for $\ell_2$ penalized problems but other penalties are not as straightforward. Similarly, the theory only holds for smooth loss functions. However, as demonstrated empirically \locod also performs well with a non-smooth loss function. 

As $n$ grows very large, the random feature matrices may become too large to communicate efficiently even when the projection dimension is very small. For these situations, there are a few simple extensions we aim to explore in future work. One possibility is to first perform row-wise random projections (c.f.\ \cite{Mahoney:2011te}) to further reduce the communication requirement. Another option is to distribute $\Xt$ according to rows and columns.

Contrary to stochastic optimization methods, the communication of \locod is limited to a single round. For fixed $\samp$, $\dims$ and $\dimsks$, the amount of communication is deterministic and can be fixed ahead of time. This can be beneficial in settings where there are additional constraints on communication (for example when different blocks of features are  distributed \emph{a priori} across different physical locations). 

Clearly with additional communication, the theoretical and practical performance of \locod could be improved. For example, \cite{zhang2012recovering} suggest an iterative dual random projection scheme which can reduce the error in Lemma \ref{thm:lowrank} exponentially. A related question for future research involves quantifying the amount of communication performed by \locod in terms of known minimax lower bounds \cite{zhang:2013inf}. 


\clearpage 
\newpage
\subsubsection*{References}
{
\small
\bibliographystyle{unsrt}
\bibliography{loco}
}

\newpage
\appendix
\begin{center}
{\Large \bf Supplementary Information for \locod: Distributing \vspace{1mm} \\
Statistical Estimation Using Random Projections}
\end{center}

\section{Supplementary Results}

Here we introduce two lemmas. The first describes the random projection construction which we use in the distributed setting. 

\begin{lem}[Summing random features]\label{lem:sumRF}
Consider the singular value decomposition $\Xt = \Ut \boldSigma \Vt\tr$ where $\Ut \in \R^{\samp\times r}$ and $\Vt \in \R^{p\times r}$ have orthonormal columns and $\boldSigma \in \R^{r\times r}$ is diagonal; $r=\text{rank}(\Xt)$. $c_0$ is a fixed positive constant. In addition to the raw features, let $\feats_k \in \R^{n\times(\dimsk + \dimsks)}$ contain random features which result from summing the $K-1$ random projections from the other workers. Furthermore, assume without loss of generality that the problem is permuted so that the raw features of worker $k$'s problem are the first $\dimsk$ columns of $\Xt$ and $\feats_k$.
Finally, let
$$
\Theta_S = \begin{bmatrix}
\It_\tau & 0\\ 
0 & \RP
\end{bmatrix}
\in \R^{p\times(\dimsk + \dimsks)}
$$
such that $\feats_k = \Xt\Theta_S.$ 

With probability at least $1-\br{\delta + \frac{\dims-\dimsk}{e^r}}$
$$
\nrm{\Vt\tr\Theta_S \Theta_S\tr\Vt - \Vt\tr\Vt} \leq \sqrt{\frac{c_0 \log(2 r /\delta) r}{\dimsks}}.
$$
\end{lem}
\begin{proof}
	See Appendix \ref{sec:sup_row}.
\end{proof}

\begin{defn} \label{def:losses}
For ease of exposition, we shall rewrite the dual problems so that we consider minimizing convex objective functions.  More formally, the original problem is then given by
\begin{equation}\label{eq:min_orig_dual}
\alphaopt = \argmin_{\boldalpha\in\R^{\samp}} \cb{D(\boldalpha) := \sum_{i=1}^\samp f^{*}_i(\alpha_i) + \frac{1}{2\samp\lambda}  \boldalpha\tr {\Kt} \boldalpha } .
\end{equation}

The problem worker $k$ solves is described by
\begin{equation}\label{eq:min_local_dual}
\alphatil = \argmin_{\boldalpha\in\R^{\samp}} \cb{ \tilde{D}_k(\boldalpha) :=  \sum_{i=1}^\samp f^{*}_i(\alpha_i) + \frac{1}{2\samp\lambda}  \boldalpha\tr \tilde{\Kt}_k \boldalpha }.
\end{equation}
Recall that $\tilde{\Kt}_k = \feats_k\feats_k\tr$, where $\feats_k$ is the concatenation of the $\dimsk$ \emph{raw} features and $\dimsks$ \emph{random} features for worker $k$.

\end{defn}
To proceed we need the following result which relates the solution of the original problem to that of the approximate problem solved by worker $k$.

\begin{lem}[Adapted from Lemma 1 \cite{zhang2014random}] \label{lem:newlem}
Let $\alphaopt$ and $\alphatil$ be as defined in Definition \ref{def:losses}. We obtain
\begin{equation} \label{eq:ineq1}
\frac{1}{\lambda}(\alphatil - \alphaopt)\tr \br{\Kt-\tilde{\Kt}_k}\alphaopt \geq \frac{1}{\lambda}(\alphatil - \alphaopt)\tr \tilde{\Kt_k}(\alphatil - \alphaopt)  .
\end{equation}
\end{lem}
\begin{proof}
	See \cite{zhang2014random}.
\end{proof}

For our main result, we rely heavily on the following variant of Theorem 1 in \cite{zhang2014random} which bounds the difference between the coefficients estimated by worker $k$, $\solnest_k$ and the corresponding coordinates of the optimal solution vector $\solnopt_k$.
\begin{lem}[Local optimization error. Adapted from \cite{zhang2014random}]\label{thm:lowrank}
 For $\rho = \sqrt{\frac{c_0 \log(2\rank/\delta) \rank}{\dimsks}}$ the following holds
$$
\nrm{\solnest_k - \solnopt_k} \leq \frac{\rho}{1-\rho}\nrm{\solnopt} 
$$
with probability at least $1-\br{\delta + \frac{\dims-\dimsk}{e^\rank}}$.
\end{lem}
The proof closely follows the proof of Theorem 1 in \cite{zhang2014random} which we restate here identifying the major differences. 
\begin{proof}

Let the quantities $\tilde{D}_k(\boldalpha)$, $ \tilde{\Kt}_k$, be as in Definition \ref{def:losses}. For ease of notation, we shall omit the subscript $k$ in $\tilde{D}_k(\boldalpha)$ and $ \tilde{\Kt}_k$ in the following.

By the SVD we have $\Xt=\Ut\boldSigma\Vt\tr$. So $\Kt = \Ut \boldSigma \boldSigma \Ut\tr$ and $\tilde{\Kt} = \Ut\boldSigma \Vt \tr \RP \RP\tr \Vt \boldSigma \Ut\tr$. We can make the following definitions
$$
\gamma^{*} = \boldSigma \Ut\tr \alphaopt, \qquad \tilde{\gamma} = \boldSigma \Ut\tr \alphatil .
$$
Defining $\tilde{\Mt} = \Vt\tr\RP\RP\tr\Vt$ and plugging these into Lemma \ref{lem:newlem} 
 we obtain
\begin{align}
(\tilde{\gamma} - \gamma^{*})\tr (\It - \tilde{\Mt}) \gamma^{*}
 \geq (\tilde{\gamma} - \gamma^{*})\tr \tilde{\Mt} (\tilde{\gamma} - \gamma^{*}) . \label{eq:ineqL}
\end{align}

We now bound the spectral norm of $\It - \tilde{\Mt}$ using Lemma \ref{lem:sumRF}. Recall that Lemma \ref{lem:sumRF} bounds the difference between a matrix and its approximation by a \emph{distributed} dimensionality reduction using the SRHT. 

Using the Cauchy-Schwarz inequality we have for the l.h.s. of \eqref{eq:ineqL}
\begin{align*}
& (\gammatil - \gammaopt)\tr \br{\It-\tilde{\Mt}} \gammaopt  
\leq  \rho\nrm{ \gammaopt}\nrm{ \gammatil - \gammaopt }
\end{align*}

For the r.h.s. of \eqref{eq:ineqL}, we can write 
\begin{align*}
&(\gammatil - \gammaopt)\tr  \tilde{\Mt} (\gammatil - \gammaopt)	
\\
& = \nrm{ \gammatil-\gammaopt}^2 - (\gammatil - \gammaopt)\tr  \br{\It-\tilde{\Mt}} (\gammatil - \gammaopt)
\\
& \geq 	\nrm{ \gammatil-\gammaopt}^2 - \rho\nrm{ \gammatil - \gammaopt }^2 
\\
& = (1-\rho) \nrm{ \gammatil - \gammaopt}^2 .
\end{align*}

Combining these two expressions and inequality \eqref{eq:ineqL} yields
\begin{align}
	(1-\rho) \nrm{ \gammatil - \gammaopt}^2
	& \leq \rho\nrm{ \gammaopt}\nrm{ \gammatil - \gammaopt } \nonumber\\
	(1-\rho) \nrm{ \gammatil - \gammaopt}
	& \leq \rho\nrm{ \gammaopt} .\label{eq:email}
\end{align}

From the definition of $\gammaopt$ and $\gammatil$ above and $\solnopt$ and $\tilde{\soln}$, respectively we have
$$
\solnopt = -\frac{1}{\lambda}\Vt \gammaopt, \qquad \tilde{\soln} = -\frac{1}{\lambda}\Vt \gammatil
$$
so $\frac{1}{\lambda}\nrm{\gammaopt} = \nrm{\solnopt}$ and $\nrm{\tilde{\soln} - \solnopt} = \frac{1}{\lambda}\nrm{\gammatil - \gammaopt}$ due to the orthonormality of $\Vt$. Plugging this into \eqref{eq:email} and using the fact that $\nrm{\solnopt - \tilde{\soln}} \geq \nrm{\solnopt_k - \solnest_k}$ we obtain the stated result.
\end{proof}

\section{Proof of Row Summing Lemma}\label{sec:sup_row}

\begin{proof}[Proof of Lemma \ref{lem:sumRF} ]

Let $\Vt_k$ contain the first $\dimsk$ rows of $\Vt$ and let $\Vt_{\minusk}$ be the matrix containing the remaining rows. Decompose the matrix products as follows
\begin{align*}
 \Vt\tr\Vt &= \Vt_k\tr\Vt_k + \Vt_{\minusk}\tr\Vt_{\minusk} \\
\quad &\text{ and } \quad \\
\Vt\tr\Theta_S \Theta_S\tr\Vt &=  \Vt_k\tr\Vt_k+ \tilde{\Vt}_k\tr\tilde{\Vt}_k
\end{align*} 

with  $\tilde{\Vt}_k\tr = \Vt_{\minusk}\tr \RP$. Then
\begin{align*}
& \nrm{\Vt\tr\Theta_S \Theta_S\tr\Vt - \Vt\tr\Vt} \\ 
& = \nrm{\Vt_k\tr\Vt_k+ \tilde{\Vt}_k\tr\tilde{\Vt}_k  - \Vt_k\tr\Vt_k - \Vt_{\minusk}\tr\Vt_{\minusk} } 
\\
& = \nrm{\Vt_{\minusk}\tr \RP \RP\tr \Vt_{\minusk} - \Vt_{\minusk}\tr\Vt_{\minusk} }.
\end{align*}
Since $\Theta_S$ is an orthogonal matrix, from Lemma 3.3 in \cite{Tropp:2010uo} and Lemma \ref{lem:summed_row}, summing $(K-1)$ independent SRHTs from $\dimsk$ to $\dimsks$ is equivalent to applying a single SRHT from $\dims-\dimsk$ to $\dimsks$. 
Therefore we can simply apply Lemma 1 of \cite{lu:2013} to the above to obtain the result.
\end{proof}

\begin{lem}[Summed row sampling]\label{lem:summed_row}
Let $\Wt$ be an $\samp \times \dims$ matrix with orthonormal columns. Let $\Wt_1,\ldots,\Wt_K$ be a balanced, random partitioning of the rows of $\Wt$ where each matrix $\Wt_k$ has exactly $\tau=\samp/ \blocks$ rows. 
Define the quantity $M:=\samp \cdot \max_{j=1,\ldots\samp}\nrm{e_j\tr\Wt}^2$. For a positive parameter $\alpha$, select the subsample size 
$$
l \cdot \blocks \geq \alpha M\log(\dims) .
$$
Let $\St_{T_k}\in\R^{l\times \tau}$ denote the operation of uniformly at random sampling a subset, $T_k$ of the rows of $\Wt_k$ by sampling $l$ coordinates from $\cb{1,2,\ldots \tau}$ without replacement. Now denote $\St \Wt$ as the sum of the subsampled rows 
$$\St \Wt = \sum_{k=1}^K\ \br{\St_{T_k}\Wt_k}.$$ 
Then
\begin{align*}
\sqrt{\frac{(1-\delta)l \cdot K }{\samp}} &\leq \sigma_{\dims}(\St \Wt) \\ 
\quad &\text{ and }  \quad  \\
\sigma_{1}(\St \Wt) &\leq \sqrt{\frac{(1+\eta)l  \cdot K}{\samp}} 
\end{align*}
with failure probability at most 
$$
\dims\cdot \sq{\frac{e^{-\delta}}{(1-\delta)^{1-\delta}}}^{\alpha \log \dims} 
+ 
\dims\cdot \sq{\frac{e^{\eta}}{(1+\eta)^{1+\eta}}}^{\alpha \log \dims}
$$
\end{lem}
\begin{proof}
Define $\wt\tr_j$ as the $j^{th}$ row of $\Wt$ and $M:=\samp\cdot\max_j \nrm{\wt_j}^2$. Suppose $K=2$ and consider the matrix 
\begin{align*}
\Gt_2: & = (\St_1\Wt_1+\St_2\Wt_2)\tr (\St_1\Wt_1+\St_2\Wt_2) 
\\
& = (\St_1\Wt_1)\tr  (\St_1\Wt_1) + (\St_2\Wt_2)\tr  (\St_2\Wt_2) \\ 
& \quad + (\St_1\Wt_1)\tr  (\St_2\Wt_2) + (\St_2\Wt_2)\tr  (\St_1\Wt_1) .
\end{align*}
In general, we can express $\Gt:=(\St\Wt)\tr(\St\Wt)$ as 
$$
\Gt := \sum_{k=1}^K \sum_{j\in T_k} \br{ \wt_j\wt_j\tr + \sum_{k'\neq k} \sum_{j'\in T_k'}  \wt_j \wt_{j'}\tr}. 
$$
By the orthonormality of $\Wt$, the cross terms cancel as $\wt_j\wt_{j'}\tr = {\bf 0}$, yielding
$$
\Gt := \br{\St\Wt}\tr\br{\St\Wt} = \sum_{k=1}^K\sum_{j\in T_k} \wt_j\wt_j\tr .
$$
We can consider $\Gt$ as a sum of $l\cdot \blocks$ random matrices 
$$
{\Xt^{(1)}_{1}},\ldots,{\Xt^{(K)}_{1}}  
,\ldots ,
{\Xt^{(1)}_l}
,\ldots,{\Xt^{(K)}_{l}}
$$ 
sampled uniformly at random without replacement from the family $\mathcal{X} := \cb{\wt_i\wt_i\tr: i=1,\ldots,\dimsk \cdot \blocks}$.

To use the matrix Chernoff bound in Lemma \ref{lem:chernoff}, we require the quantities $\mu_{\min}$, $\mu_{\max}$ and $B$. Noticing that $\lambda_{\max}(\wt_j\wt_j\tr) = \nrm{\wt_j}^2 \leq \frac{M}{\samp}$, we can set $B \leq M/\samp$.

Taking expectations with respect to the random partitioning ($\bE_{P}$) and the subsampling within each partition ($\bE_{S}$),
using the fact that columns of $\Wt$ are orthonormal we obtain 
$$
\bE \sq{{\Xt^{(k)}_{1}}} =  \bE_{P} \bE_{S} \Xt_1^{(k)} = \frac{1}{\blocks}  \frac{1}{\tau} \sum_{i=1}^{\blocks \tau} \wt_i\wt_i\tr = \frac{1 }{n}\Wt\tr\Wt = \frac{1 }{n}\It
$$
Recall that we take $l$ samples in $\blocks$ blocks so we can define
$$
\mu_{\min} = \frac{l\cdot K}{\samp} \qquad \text{ and } \qquad \mu_{\max} = \frac{l\cdot K}{\samp} .
$$
\\
Plugging these values into Lemma \ref{lem:chernoff}, the lower and upper Chernoff bounds respectively yield
\begin{align*}
& \mathbb{P} \cb{ \lambda_{\min}\br{\Gt} \leq (1-\delta)\frac{l \cdot  \blocks}{\samp}} \\ 
& \quad  \leq \dims \cdot \sq{ \frac{e^{-\delta}}{(1-\delta)^{1-\delta}}}^{l\cdot \blocks/M} \text{ for } \delta \in [0,1), \text{ and}
\\
\\
& \mathbb{P} \cb{ \lambda_{\max}\br{\Gt} \geq (1+\delta)\frac{l \cdot  \blocks}{\samp}} \\
& \quad \leq \dims \cdot \sq{ \frac{e^{\delta}}{(1+\delta)^{1+\delta}}}^{l\cdot \blocks/M} \text{ for } \delta \geq 0 .
\end{align*}

Noting that $\lambda_{\min}(\Gt) = \sigma_{\dims}(\Gt)^2$, similarly for $\lambda_{\max}$ and using the identity for $\Gt$ above obtains the desired result.
\end{proof}

For ease of reference, we also restate the Matrix Chernoff bound from \cite{Tropp:2010uo,Tropp:2010vm} but defer its proof to the original papers.

\begin{lem}[Matrix Chernoff from \cite{Tropp:2010uo}] \label{lem:chernoff}
Let $\mathcal{X}$ be a finite set of positive-semidefinite matrices with dimension $\dims$, and suppose that
$$
\max_{\At\in\mathcal{X}} \lambda_{\max}(\At) \leq B
$$
Sample $\{ \At_1,\ldots,\At_l \}$ uniformly at random from $\mathcal{X}$ without replacement. Compute 
$$
\mu_{\min} = l \cdot \lambda_{\min}(\bE \Xt_1) \qquad \text{and } \qquad \mu_{\max} = l \cdot \lambda_{\max}(\bE \Xt_1) 
$$
Then
\begin{align*}
& \mathbb{P} \cb{ \lambda_{\min}\br{\sum_i \At_i} \leq (1-\delta)\mu_{\min}} \\
& \quad \leq \dims \cdot \sq{ \frac{e^{-\delta}}{(1-\delta)^{1-\delta}}}^{\mu_{\min}/B} \text{ for } \delta \in [0,1), \text{ and}
\\
\\
& \mathbb{P} \cb{ \lambda_{\max}\br{\sum_i \At_i} \geq (1+\delta)\mu_{\max}} \\
& \quad \leq \dims \cdot \sq{ \frac{e^{\delta}}{(1+\delta)^{1+\delta}}}^{\mu_{\max}/B} \text{ for } \delta \geq 0 .
\end{align*}
\end{lem}

\section{Supplementary Material for Section~\ref{sec:implementation}}\label{sec:supp_exp}

\vspace{-.5cm}
\begin{minipage}{0.55\textwidth}
\begin{algorithm}[H]
\caption{\locod \xspace-- cross validation \label{alg:cv}}
	\algorithmicrequire\; Data: $\Xt$, $\y$, no.\ workers: $\blocks$, no.\ folds: $v$ \\ Parameters: $\dimsks$, $\lambda_1, \ldots \lambda_l$
  \begin{algorithmic}[1]
  \STATE Partition $\{p\}$ into $\blocks$ subsets of equal size $\dimsk$ and distribute feature vectors in $\Xt$ accordingly over $\blocks$ workers.
  \STATE Partition $\{n\}$ into $v$ folds of equal size.
  \FOR{\textbf{each} fold $f$}
  \STATE Communicate indices of training and test points.
    \FOR{\textbf{each} worker $k\in\{1,\ldots\blocks\}$ \textbf{in parallel}}
        \STATE Compute and send $ \Xt_{k,f}^{train} \RP_{k,f}$.
        \STATE Receive random features and construct $\feats_{k,f}^{train}$. 
		\FOR{\textbf{each} $\lambda_j \in \{\lambda_1, \ldots \lambda_l\}$}
        \STATE $\alphatil_{k,f,\lambda_j} \leftarrow \texttt{LocalDualSolver}(\feats_{k,f}^{train},\y^{train}_f,\lambda_j) $
        \STATE $\solnest_{k,f,\lambda_j} = -\frac{1}{\samp \lambda_j} {\Xt_{k,f}^{train}}\tr\alphatil_{k,f,\lambda_j}$
		\STATE $ \hat{\y}_{k,f,\lambda_j}^{test} =  \Xt_{k,f}^{test} \solnest_{k,f,\lambda_j} $   
		\STATE Send  $ \hat{\y}_{k,f,\lambda_j}^{test}$ to driver.
        \ENDFOR
    \ENDFOR
  
	\FOR{\textbf{each} $\lambda_j \in \{\lambda_1, \ldots \lambda_l\}$}
	\STATE Compute $ \hat{\y}_{f,\lambda_j}^{test} = \sum_{k = 1}^\blocks \hat{\y}_{k,f,\lambda_j}^{test}$.
    \STATE Compute $\text{MSE}_{f, \lambda_j}^{test}$ with $\hat{\y}_{f,\lambda_j}^{test} $ and $\y_{f}^{test} $.
	\ENDFOR	
	
  \ENDFOR
  
  \FOR{\textbf{each} $\lambda_j \in \{\lambda_1, \ldots \lambda_l\}$}
  	\STATE Compute $\text{MSE}_{\lambda_j} = \frac{1}{v} \sum_{f = 1}^v \text{MSE}_{f, \lambda_j}$.
  \ENDFOR
	  \end{algorithmic}
  \algorithmicensure\; Parameter $\lambda_j$ attaining smallest $\text{MSE}_{\lambda_j}$
\end{algorithm} 
\end{minipage}

\newcommand{\ra}[1]{\renewcommand{\arraystretch}{#1}}
\ra{1.1}
\begin{table}[!tbp]
\small
\begin{center}
\begin{tabular}{llll}
\hline\hline
\multicolumn{1}{c}{Algorithm}&\multicolumn{1}{c}{K}&\multicolumn{1}{c}{TEST MSE}&\multicolumn{1}{c}{TRAIN MSE}\tabularnewline
\hline
\locod 0.5&12&0.0343 (3.75e-03) &0.0344 (2.59e-03) \tabularnewline
\locod 0.5&24&0.0368 (4.22e-03) &0.0344 (3.05e-03) \tabularnewline
\locod 0.5&48&0.0328 (3.97e-03) &0.0332 (2.91e-03) \tabularnewline
\locod 0.5&96&0.0326 (3.13e-03) &0.0340 (2.67e-03) \tabularnewline
\locod 0.5&192&0.0345 (3.82e-03) &0.0345 (2.69e-03) \tabularnewline
\hline
\locod 1&12&0.0310 (2.89e-03) &0.0295 (2.28e-03) \tabularnewline
\locod 1&24&0.0303 (2.87e-03) &0.0307 (1.44e-03) \tabularnewline
\locod 1&48&0.0328 (1.92e-03) &0.0329 (1.55e-03) \tabularnewline
\locod 1&96&0.0299 (1.07e-03) &0.0299 (7.77e-04) \tabularnewline
\hline
\locod 2&12&0.0291 (2.16e-03) &0.0280 (6.80e-04) \tabularnewline
\locod 2&24&0.0306 (2.38e-03) &0.0279 (1.24e-03) \tabularnewline
\locod 2&48&0.0285 (6.11e-04) &0.0293 (4.77e-04) \tabularnewline
\hline
\cocoap&12&0.0282 (4.25e-18) &0.0246 (2.45e-18) \tabularnewline
\cocoap&24&0.0278 (3.47e-18) &0.0212 (3.00e-18) \tabularnewline
\cocoap&48&0.0246 (6.01e-18) &0.0011 (1.53e-19) \tabularnewline
\cocoap&96&0.0254 (5.49e-18) &0.0137 (1.50e-18) \tabularnewline
\cocoap&192&0.0268 (1.23e-17) &0.0158 (6.21e-18) \tabularnewline
\hline
\end{tabular}\end{center}
\caption{\small Dogs vs Cats data: Normalized training and test MSE: mean and standard deviations (based on 5 repetitions). \label{tab:mse_numbers}}
\end{table}

\end{document}